\documentclass{amsart}

\usepackage{hyperref}
\usepackage{amsmath, amssymb,amsthm}
\usepackage{stmaryrd}
\usepackage{epsfig,calc, graphicx}
\usepackage{wasysym}
\usepackage[usenames]{color}
\usepackage{graphicx}
\usepackage{subcaption}
\usepackage{multirow}
\usepackage{booktabs}
\usepackage{hyperref}
\usepackage{caption}
\usepackage{todonotes}
\usepackage{array}
\usepackage{tabulary}
\usepackage{algorithmic}
\usepackage[boxed]{algorithm2e}
\usepackage{ulem}
\newlength\myindent
\newcommand\bindent{%
  \begingroup
  \setlength{\itemindent}{\myindent}
  \addtolength{\algorithmicindent}{\myindent}
}
\newcommand\eindent{\endgroup}
\DontPrintSemicolon
\SetKwProg{Fn}{Function}{:}{}
\SetKwComment{Comment}{$\triangleright$\ }{ }
\usepackage{subcaption}

\DeclareMathOperator{\sign}{sgn}
\DeclareMathOperator{\supp}{supp}
\DeclareMathOperator*{\argmin}{arg\,min}
\DeclareMathOperator*{\argmax}{arg\,max}
\newcommand{\R}{\mathbb{R}}
\newcommand{\N}{\mathbb{N}}
\newcommand{\Id}{\mathbb{I}}

\theoremstyle{plain}
\newtheorem{theorem}{Theorem}

\newtheorem{lemma}[theorem]{Lemma}

\newtheorem{remark}[theorem]{Remark}

\newtheorem{postulate}{Assumption} 

\theoremstyle{definition}

\newcommand{\norm}[1]{\left\lVert #1 \right\rVert}

\newcommand*{\mailto}[1]{\href{mailto:#1}{\nolinkurl{#1}}}

\begin{document}

\title[Adaptive multi-penalty regularization]{Adaptive multi-penalty regularization based on a generalized Lasso path}

\date{October 11, 2017}

\author[M.~Grasmair]{Markus Grasmair}
\address{Department of Mathematical Sciences, Norwegian University of Science and Technology, N-7491 Trondheim, Norway}
\email{\mailto{markus.grasmair@ntnu.no}}

\author[T.~Klock]{Timo Klock}
\address{Section for Computing and Software, Simula Research Laboratory AS, 1364 Fornebu, Norway}
\email{\mailto{timo@simula.no}}

\author[V.~Naumova]{Valeriya Naumova}
\address{Section for Computing and Software, Simula Research Laboratory AS, 1364 Fornebu, Norway}
\email{\mailto{valeriya@simula.no}}

\begin{abstract}
For many algorithms, parameter tuning remains a challenging and critical task, which becomes tedious and infeasible in a multi-parameter setting. Multi-penalty regularization, successfully used for solving undetermined sparse regression of problems of unmixing type where signal and noise are additively mixed, is one of such examples. 
In this paper, we propose a novel algorithmic framework for an adaptive parameter choice in multi-penalty regularization with a focus on the correct support recovery. Building upon the theory of regularization paths and algorithms for single-penalty functionals, we extend these ideas to a multi-penalty framework by providing an efficient procedure for the construction of regions containing structurally similar solutions, i.e., solutions with the same sparsity and sign pattern, over the whole range of parameters. Combining this with a model selection criterion, we can choose regularization parameters in a data-adaptive manner.
Another advantage of our algorithm is that it provides an overview on the solution stability over the whole range of parameters. This can be further exploited to obtain additional insights into the problem of interest. We provide a numerical analysis of our method and compare it to the state-of-the-art single-penalty algorithms for compressed sensing problems in order to demonstrate the robustness and power of the proposed algorithm.
\end{abstract}


\maketitle

\section{Introduction}

\subsection{Multi-penalty regularization for unmixing problems}

Support recovery of a sparse high-dimensional signal still remains a challenging  task from both theoretical and practical perspectives for a variety of applications from harmonic analysis, signal processing, and compressed sensing, see \cite{zhang2009, tropp04} and references therein. Indeed, provided with the signal's support, the signal entries can be easily recovered with optimal statistical rate \cite{wainwright}. Therefore, support recovery has been a topic of active and fruitful research in the last years \cite{shen, osher16, bouchot}. One typically considers linear observation model problems of the form  
 \begin{equation}
\label{lineq}
	A u^\dag  + \delta = y,
\end{equation}
where  $u^\dag \in \R^n$ is the unknown signal  with only a few non-zero entries, $A \in \R^{m \times n}$ is the linear measurement matrix, $\delta  \in \R^m$ is a noise vector affecting the measurements, and $y \in \R^m$ is the result of the measurement, usually $m\ll n$. However, in a more realistic scenario, it is very common that noise is present not only in the measurements but also in the signal itself. In this context we study the impact of the noise folding phenomenon, a situation commonly encountered in compressed sensing \cite{Arias-castro11yeldar, aeron} due to subsampling in the presence of signal noise.  Mathematically we can formalize this scenario by the model
\begin{equation}
\label{unmixing_noise_folding}
	A (u^\dag + v^\dag) + \delta = y,
\end{equation}
where $v^\dag \in \R^n$ is the signal noise and $\delta$ is again a measurement noise. The impact of signal noise on the exact support recovery of the original signal was reported and analysed in \cite{Arias-castro11yeldar, aeron}. Essentially, the authors \cite{aeron}  claim that the exact support recovery is possible when the number of measurements $m$ scales linearly with $n$, leading to a poor compression performance in such cases. 
These negative results can be circumvented by using, for instance, decoding procedures based on multi-penalty regularization as proposed in \cite{arfopeXX, naumpeter, Markus}. There, the authors provided theoretical and numerical pieces of evidence of improved performance of the multi-penalty regularization schemes for the solution of \eqref{unmixing_noise_folding}, especially with respect to support identification, as compared to their single-penalty counterparts.
Inspired by these results, in this paper  we consider the following minimization problem 
\begin{equation}
	\label{eq:multi}
J (u,v) : = \| A(u+v)-y \|^2 +  \alpha \| u\|_{\ell_1} + \beta  \| v\|_{\ell_2}^{2} \rightarrow \min\limits_{u, v},
\end{equation} 
where $\| u\|_{\ell_p} =( \sum_i |u_i |^p)^{1/p}$, $p = 1,2$, denotes the $\ell_p-$norm and $\alpha,\beta > 0$ are regularization parameters. We will denote any solution of \eqref{eq:multi} by $(u_{\beta,\alpha}, v_{\beta,\alpha}).$ The functional \eqref{eq:multi} uses a non-smooth term  $\| \cdot \|_{\ell_1}$ for promoting sparsity of $u$ and a quadratic penalty term for modeling the noise.

One of the main ingredients for the optimal performance of multi-penalty regularization is an appropriate choice of multiple regularization parameters, ideally in a data-adaptive manner. This issue has not been studied so far for this type of problems.
Instead, the earlier works rather rely on brute-force approaches to choose parameters.

To circumvent difficulties related to the regularization parameters choice,  in this paper we propose a two-step procedure for reconstructing the support of the signal of interest $u^\dag$. First, we compute in a cost-efficient way all possible sufficiently sparse supports and sign patterns of the signal attainable from $y$ by means of (\ref{eq:multi}) for different regularization parameters. Then we employ standard regression methods for estimating a vector $u$ that provides the best explanation of the datum $y$ for each found support and sign pattern. As such, we have a complete overview of the solution behaviour over the range of parameters, without imposing any a priori assumption. This allow us to choose regularization parameters  in a data-adaptive manner.

\subsection{Related work}

The formulation of multi-penalty functionals of the type  (\ref{eq:multi})  is not novel. Especially in image and signal analysis, multi-penalty regularization has been presented and analysed in seminal papers by Meyer \cite{Meyer} and Donoho \cite{MR1872845, MR997928}. We refer to \cite{naumpeter, demole} for a thorough overview of the work on multi-penalty regularization in image and signal processing communities. 

Multi-penalty regularization functional \eqref{eq:multi} is equivalent to Huber-norm regularization \cite{huber1964robust}, which is commonly used in image super resolution optimization problems and computer-graphics problems.  The recent work \cite{Zadorozhnyi2016} provides an upper bound on the statistical risk of Huber-regularized linear models by means of the Rademacher complexity. It also provides empirical evidences that the support vector machine with Huber regularizer outperforms its single-penalty counterparts and Elastic-Net  \cite{RSSB:RSSB503} on a wide range of benchmark datasets. In this paper, we pursue a different goal, where  we focus on the support recovery, rather than classification or regression problems.
In this context, systematic theoretical and numerical investigations have only been considered recently \cite{arfopeXX, naumpeter, Markus}.

At the same time, one of the key ingredients for optimal regularization is an adaptive parameter choice. This is generally a challenging task,  especially so in a multi-parameter setting, and it has not yet been studied in the context of sparse support recovery. Grasmair and Naumova present a solution that is conceptually closest to the present paper and also provides bounds on admissible regularization parameters \cite{Markus}. However, these bounds require accurate knowledge about the nature of the signal and the noise, which is rarely available in practice.

\subsection{Contributions}
The paper provides an efficient algorithm for the identification of possible parameter regions leading to structurally similar solutions for the multi-penalty functional (\ref{eq:multi}), i.e.,  solutions with the same sparsity and sign pattern. The main advantage of the proposed algorithm is that, without strong a priori assumptions on the solution, it provides an overview on the solution stability over the whole parameter range. This information can  be exploited for obtaining additional insights into the problem of interest. Furthermore, by combining the algorithm with a simple region selection criterion, regularization parameters can be chosen in a fully adaptive data-driven manner. In Section \ref{se:numerical_experiments}, we provide extensive numerical evidence that our combined algorithm shows a better performance in terms of support recovery when compared to its closest counterparts
such as Lasso and pre-conditioned Lasso (pLasso),
while still having a reasonable computational complexity. 
However, it is worthwhile to stress that our method recovers the solution $u^\dag$ and $v^\dag$ for the entire range of parameters rather than only at specific parameter combinations.

\subsection{Organization of the paper}
The rest of the paper is organized as follows. Section \ref{se:mpr} contains the complete problem set-up, explaining
multi-penalty regularization, and provides in-depth overview of the relevant existing literature.
In Section \ref{sec:tiles}, we introduce tilings as a conceptual framework for studying the  structure of the support of the solution. Section \ref{sec:algorithm} presents an  algorithmic approach to the construction of the complete support tiling.  In Section \ref{se:heuristics} we then discuss the actual realization of our algorithm, providing also  its complexity analysis. Numerical experiments are carried out in Section \ref{se:numerical_experiments}.  Finally, Section~\ref{sec:discussion} offers a snapshot of the main contributions and points out open questions and directions for future work.

\subsection{Notation}
We provide a short overview of the standard notations used in this paper. The true solution $u^\dag$ of the unmixing problem~(\ref{unmixing_noise_folding}) is called $k$-sparse if it has at most $k$ non-zero entries, i.e., $| I | = \# \mathop{{\rm supp}}(u^\dagger) \leq k$,
where $I := \mathop{{\rm supp}}(u^\dag) := \bigl\{i: u^\dag_i \neq 0\}$ denotes the support of $u^\dag$. 
For a matrix $A$, we denote its transpose by $A^T.$ 
The restriction of the matrix $A$ to the columns indexed by $I$ is denoted by $A_I$, i.e., $A_I = (a_{i_1}, \ldots, a_{i_k}), i_k \in I$. The matrix $\Id$ denotes the identity matrix of relevant size.
The sign function $\sign$ is interpreted as  the set valued function $\sign(t) =1$ if $t>0$, $\sign(t) = -1$ if $t<0$ and $\sign(t) = [-1,1]$ if $t = 0$, applied componentwise to the entries of the vector. 

\section{Multi-parameter regularization}
\label{se:mpr}
In this section, after formally introducing the multi-penalty regularization for solving \eqref{unmixing_noise_folding} and relevant known results, we discuss a possible parameter choice for multi-penalty functional \eqref{eq:multi} based on the Lasso path \cite{tibshirani2013lasso}. We show that an extension of the single-penalty Lasso path by partial discretization of the parameter space to the multi-parameter case can have difficulties in capturing the exact support, especially when the solution is very sensitive with respect to the parameters change.        

Inspired by recent theoretical results \cite{Markus}, we  propose to solve the unmixing problem \eqref{unmixing_noise_folding} using multi-penalty Tikhonov functional  \eqref{eq:multi}, where  $\alpha$ and $\beta$ are regularization parameters. 
The starting point for the approach proposed in this paper is the observation that \eqref{eq:multi} reduces to standard $\ell_1$-regularization but where the measurement matrix and the datum are additionally tuned by the second regularization parameter $\beta$. As it has been shown in \cite{Markus}, this modification leads to a superior performance over standard sparsity-promoting regularization.
The main result to that end is the following.

\begin{lemma}\label{le:single}
  The pair $(u_{\beta, \alpha},v_{\beta, \alpha})$ solves~(\ref{eq:multi})
  if and only if
  \[
  v_{\beta, \alpha} = (\beta + A^TA)^{-1}(A^T y - A^TA u_{\beta,\alpha})
  \]
  and $u_{\beta, \alpha}$ solves the optimisation problem
  \begin{equation}\label{eq:minu}
  \frac{1}{2}\|B_\beta u - y_\beta\|^2 + \alpha\|u\|_1 \to \min
  \end{equation}
  with
  \[
  B_\beta = \Bigl(I + \frac{AA^T}{\beta}\Bigr)^{-1/2}A
  \]
  and
  \[
  y_\beta = \Bigl(I+\frac{AA^T}{\beta}\Bigr)^{-1/2}y.
  \]
\end{lemma}

In the following, we will assume that~\eqref{eq:multi}
always has a unique solution within the considered parameter range.

\begin{postulate}\label{po:unique}
  For all parameters $\beta$ and $\alpha$ that are considered in the following,
  the optimization problem~\eqref{eq:multi} has a unique solution
  $(u_{\beta, \alpha},v_{\beta, \alpha})$.
\end{postulate}

As a consequence of this assumption, we have the following:

\begin{lemma}\label{le:cont}
  If Assumption~\ref{po:unique} holds 
  then $u_{\beta, \alpha}$ and $v_{\beta, \alpha}$ depend continuously
  on the parameters $\alpha$, $\beta > 0$.
\end{lemma}

\begin{proof}
  Since $u_{\beta, \alpha}$ is a solution of the single parameter
  problem 
  \[
  \frac{1}{2}\norm{B_\beta u - y_\beta}^2 + \alpha\norm{u}_{1} \to \min,
  \]
  it is uniquely characterized by the Karush-Kuhn-Tucker (KKT) conditions
  \begin{equation}\label{eq:KKT}
    \begin{aligned}
      B_{\beta,i}^T(y_\beta - B_\beta u) & = \alpha \sign(u_i) &&\text{if
      } u_i \neq 0,\\
|B_{\beta,i}^T(y_\beta - B_\beta u)| &\le \alpha &&\text{if }
      u_i = 0.
    \end{aligned}
  \end{equation}
  Now note that $B_\beta$ and $y_\beta$ depend continuously on $\beta
  > 0$, which implies that also the KKT conditions change continuously
  with both $\alpha$ and $\beta$. Because of the uniqueness of
  $u_{\beta, \alpha}$, its continuous dependence on the parameters
  follows.
\end{proof}

The behaviour of the  $\ell_1$-regularization is by now fairly well understood \cite{CPA:CPA20350} and there are well-established  approaches  especially in statistical learning literature for calculating $\ell_1-$regularized solution as a function of $\alpha$ \cite{tibshirani2013lasso}, i.e., for finding the solution for all $\alpha \in [0, \infty]$.  The existing algorithms are based on the observation that the solution path of the $\ell_1-$regularized or, as it is called in statistics, Lasso problem, is piecewise-linear in each component of the solution \cite{rozhu07}. 
Our goal in this paper is to provide an efficient procedure for tracking the behaviour of $u_{\beta, \alpha}$ in terms of the support and sign pattern as a function of both regularization parameters. We will assume that the regularization parameter $\beta$ is defined in a finite interval $[\beta_{\min}, \beta_{\max}],$ which ideally should be chosen depending on a problem at hand. For the sake of self-containedness, we present the Lasso path algorithm from  a conceptual point of view, and we refer to \cite{tibshirani2013lasso} for rigorous arguments for its correctness. 

\paragraph{The Lasso path} The algorithm is essentially given by an iterative or inductive verification of the optimality conditions \eqref{eq:KKT}, starting with $\alpha_0 = \infty $ for fixed $\beta$, so that solution \eqref{eq:minu} is trivial. As the parameter $\alpha$ decreases, the algorithm computes $u_{\beta,\alpha}$ that is piecewise linear and continuous as a function of $\alpha.$ 
Each knot $\alpha^{(k)}$ in this path corresponds to an iteration of the algorithm, in which the update direction of the solution $u_{\beta,\alpha}$ is altered in order to satisfy the optimality conditions.
 
The support of $u_{\beta, \alpha}$ changes only at the knot points: as $\alpha$ decreases, each knot $ \alpha^{(k)}$ corresponds  to entries added to or deleted from the support of the solution. Such a model is attractive and efficient because it allows to generate the whole regularization path $u_{\beta, \alpha}, ~0\leq \alpha\leq \infty,$ simply by sequentially finding the knots $\alpha^{(k)}$ and calculating the solution at those points. For fixed $\beta$, these knots can be computed explicitly, see Lemma \ref{le:lassopath} below.

\paragraph{The discretized Lasso path for multi-penalty regularization} 
A proper choice of $\beta$ is essential for a good performance of the algorithm \eqref{eq:multi}.  The most straightforward way of extending the Lasso path algorithm to the multiple parameter setting is to discretize the range of the $\beta$ parameters and then solve the single-penalty problem (\ref{eq:minu}) that we obtain for each parameter $\beta$. 

Specifically, we choose an upper bound $\beta_{\max}$ and a lower bound $\beta_{\min}$ of possible values of $\beta$, and discretization points $\beta_{\min} = \beta_0 < \beta_1 < \ldots < \beta_N = \beta_{\max}$. 
For each $\beta_i$ we can use the Lasso path algorithm in order to compute the solutions of the multi-penalty regularization problem with $\beta_i$ and all parameters $\alpha$ up to a predefined sparsity level.  That is, we start for each $\beta_i$ with some sufficiently large $\alpha^{(0)}(\beta_i)$ such that the corresponding index set is $I^{(0)}(\beta_i) := \emptyset$, and the sign pattern 
$\sigma^{(0)}(\beta_i) := \emptyset$. Then we inductively compute the knot points $\alpha^{(k)}(\beta_i)$ where, as discussed above, indices either enter or leave the
support of the solution $u_{\beta_i,\alpha}$. This allows us to update the current support and sign pattern according to the following result:

\begin{lemma}[Lasso path algorithm for~\eqref{eq:minu}, see \cite{tibshirani2013lasso}]\label{le:lassopath}
Assume that $\alpha^{(k)}(\beta_i)$ has already been computed and
the solution $u_{\beta_i, \alpha}$ has the support $I$ and sign pattern $\sigma\in\{\pm 1\}^I$
for values of $\alpha$ slightly smaller than $\alpha^{(k)}(\beta_i)$. We define
for every $1 \le j \le n$ 
\begin{equation}\label{eq:talpha}
\tilde{\alpha}_{j}(I,\sigma,\beta_i) :=
\left\{
 \begin{array}{l@{\qquad}l}
  \frac{B_{\beta_i,j}^T(\Id - B_{\beta_i,I}(B_{\beta_i,I}^TB_{\beta_i,I})^{-1}B_{\beta_i,I}^T)y_{\beta_i}}{\gamma - B_{\beta_i,j}^TB_{\beta_i,I}(B_{\beta_i,I}^T B_{\beta_i,I})^{-1}\sigma}
  & \textrm{ if } j \not \in I\\
  \frac{\bigl((B_{\beta_i,I}^TB_{\beta_i,I})^{-1}B_{\beta_i,I}^Ty_{\beta_i}\bigr)_j}{\bigl((B_{\beta_i,I}^TB_{\beta_i,I})^{-1}\sigma\bigr)_j} & \textrm{ if } j \in I.
\end{array}\right.
\end{equation}
Here $\gamma = \sign(B_{\beta_i,j}^T(\Id - B_{\beta_i,I}(B_{\beta_i,I}^TB_{\beta_i,I})^{-1}B_{\beta_i,I}^T)y_{\beta_i}),$ unless $j$ was contained in the support of the solution at the previous step. In this case, we take $\gamma$ with an opposite sign to the enumerator of the first case in \eqref{eq:talpha}.

Then we have with
\[
\mathcal{K}_k(\beta_i)
:= \bigl\{j : \tilde{\alpha}_{j}(I^{(k)}(\beta_i),\sigma^{(k)}(\beta_i),\beta_i) < \alpha^{(k)}(\beta_i)\bigr\}
\]
that
\begin{equation}\label{eq:alpha_tau0}
\alpha^{(k+1)}(\beta_i)
= \max_{j\in\mathcal{K}_k(\beta_i)} \tilde{\alpha}_{j}(I^{(k)}(\beta_i),\sigma^{(k)}(\beta_i),\beta_i).
\end{equation}

The updated index set $I^{(k+1)}(\beta_i)$ is formed from $I^{(k)}(\beta_i)$ by adding all indices $j \not\in I^{(k)}(\beta_i)$ for which the maximum in~(\ref{eq:alpha_tau0}) is attained, or removing all indices $j\in I^{(k)}(\beta_i)$ for which the maximum in~(\ref{eq:alpha_tau0}) is attained. In addition, the signs $\sigma_j$ corresponding to the newly added indices are precisely the values of the corresponding $\gamma$ in~(\ref{eq:talpha}).
\end{lemma}

\begin{proof}
  This is nothing else than the usual Lasso path algorithm applied to
  the equivalent single parameter problem~\eqref{eq:minu}.
  See~\cite{tibshirani2013lasso} for more details.
\end{proof}

Provided that the discretization of the parameters $\beta$ is sufficiently fine, this method is capable of computing a reasonable approximation of the
complete tiling over the parameter space and, in particular, of identifying all possible supports and sign patterns of the regularized solution $u_{\beta,\alpha}$.

However, for certain configurations it can happen that the support of the correct solution $u^\dagger$ is achievable only for a small range of $\beta$-parameters.
An example of such a situation is presented in Figure~\ref{fi:smalltile}. Thus, a very fine discretization of the parameter range will  be needed in order to capture all possible supports. Moreover, in practice, very often even the support size is unknown and must be chosen according to some heuristic principle. In this situation, having an overview of the solution behaviour over the whole range of parameters rather than solution at discrete points could provide some additional hints for such a choice.

\begin{figure}[h]
  \centering
     \includegraphics[width=0.8\linewidth]{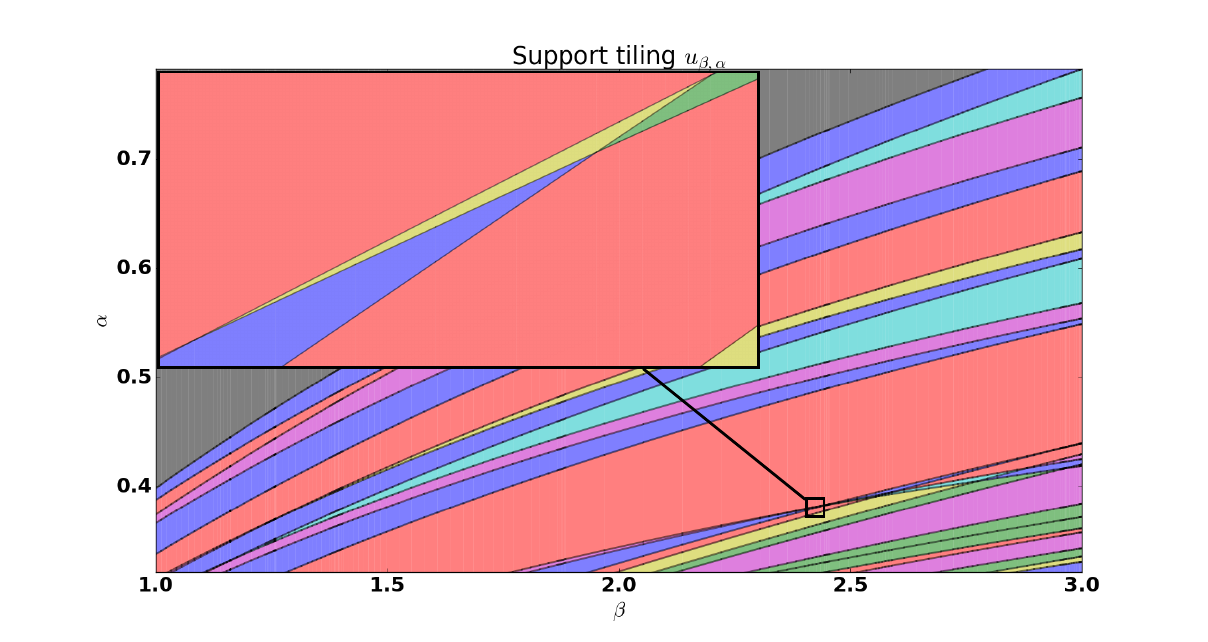}
  \caption{\label{fi:smalltile} Part of the parameter space detailing
    the different solutions. Each of the different tiles corresponds
    to a different support or sign pattern of the solution
    $u_{\beta,\alpha}$. Note in particular the small details indicating
    that certain supports can only be obtained for very specific
    parameter values. In this particular case, the yellow tile
    in the center of the zoomed region describes the parameters leading to the exact support recovery.}
\end{figure}

\begin{remark} \label{re:plasso}
It is worthwhile to mention that multi-penalty regularization can be interpreted as an interpolation between a standard and pre-conditioned \cite{jiaro15} Lasso. The latter computes approximations $u_\gamma$ of the solution of the equation
$A u = y$ by solving the optimization problem
\[
\frac12 \| F Au - Fy \|^2 + \gamma \| u \|_1 \rightarrow \min_u,
\] 
where $F = U \Sigma^{\dag} U^T$ and $A$ has the singular value decomposition  $A = U \Sigma V^T,$ where $\Sigma^{\dag}$ denotes 
the pseudo-inverse of $\Sigma$, that is, the diagonal matrix with non-zero entries equal to the inverse of the
non-zero entries of $\Sigma$.

Indeed, for $\beta \rightarrow 0$ the solution of the multi-penalty regularization $u_{\beta, \alpha}$  with $\alpha = \beta \gamma$ converges to the solution of the pre-conditioned Lasso $u_\gamma.$ This can be seen by noting that \eqref{eq:minu} is equivalent to the following minimization problem
  \begin{equation}\label{eq:minu_eq}
  \frac{1}{2}\|\hat B_\beta u - \hat y_\beta\|^2 + \frac{\alpha}{\beta}\|u\|_1 \to \min,
  \end{equation}
with 
$$
\hat B_\beta = (\beta I + A A^*)^{-1/2} A 
$$
and
$$
\hat y_\beta = (\beta I + A A^*)^{-1/2} y.
$$
As $\beta \rightarrow 0,$ we have $\hat B_\beta \rightarrow FA$ and $\hat y_\beta  \rightarrow Fy.$ By $\Gamma-$convergence, it follows that the minimizers also converge, i.e., $u_{\beta, \beta\gamma} \rightarrow u_\gamma.$ 
Conversely, for  $\beta = \infty,$ multi-penalty regularization \eqref{eq:multi} is equivalent to a standard Lasso. 

Therefore, with a properly chosen parameter $\beta$, one can expect that multi-penalty regularization combines the advantages of both regularization methods and mitigates their drawbacks. In particular, it is known \cite{jiaro15} that the pLasso is less robust to measurement noise and to ill-conditioned measurement operators compared to the standard Lasso; whereas the standard Lasso  has problems identifying the correct support for problems of unmixing type as exemplified in Section \ref{se:numerical_experiments}.
\end{remark}

\section{Support tiling}
\label{sec:tiles}

Instead of reconstructing the solution at fixed parameters $\beta_i$, we approach the extension of the Lasso path algorithm for multi-penalty regularization by constructing regions or, the so-called, tilings, containing structurally similar solutions of \eqref{eq:multi}, i.e., solutions with the same sparsity and sign pattern, while still preserving simplicity in calculations.
 In this section, we introduce this concept, providing some geometrical interpretations for ease of understanding together with necessary notation.

We denote by
\[
I_{\beta,\alpha} := \supp(u_{\beta,\alpha})
\]
the support of the regularized solution $u_{\beta,\alpha}$,
and by
\[
\sigma_{\beta,\alpha} := (\sign(u_{\beta,\alpha,i}))_{i\in I_{\beta,\alpha}} \subset \{\pm 1\}^{I_{\beta,\alpha}}
\]
the sign pattern of $u_{\beta,\alpha}$.
Given $\alpha$, $\beta > 0$, we then define
\[
\tilde{\tau}_{\beta,\alpha}
:= \{(\tilde{\beta},\tilde{\alpha}) \in \R^2_{>0}: I_{\tilde{\beta},\tilde{\alpha}} = I_{\beta,\alpha}
\text{ and } \sigma_{\tilde{\beta},\tilde{\alpha}} = \sigma_{\beta,\alpha} \},
\]
that is, $\tilde{\tau}_{\beta,\alpha}$ contains all parameters
that lead to the same support and sparsity pattern
for the reconstruction of $u$.
Additionally, we denote by $\tau_{\beta,\alpha} \subset \R^2_{>0}$
the connected component
of the set $\tilde{\tau}_{\beta,\alpha}$ that contains $(\beta,\alpha)$.
Then the family
\[
\mathcal{T} := \{\tau_{\beta,\alpha}: (\beta,\alpha) \in \R^2_{>0} \}
\]
forms a tiling of $\R^2_{>0}$.

Given a tile $\tau \in \mathcal{T}$, we now denote
\[
I_{\tau} := I_{\beta,\alpha}
\text{ and }
\sigma_{\tau} := \sigma_{\beta,\alpha}
\qquad\text{ for any } (\beta,\alpha) \in \tau,
\]
the common support and sign pattern of the reconstructions
of $u$ on the tile $\tau$.
Moreover, we define
\[
\begin{aligned}
\beta^-(\tau) &:= \inf\{\beta > 0: \text{ there exists } \alpha > 0 \text{ with } (\beta,\alpha) \in \tau\},\\
\beta^+(\tau) &:= \sup\{\beta > 0: \text{ there exists } \alpha > 0 \text{ with } (\beta,\alpha) \in \tau\},
\end{aligned}
\]
the left and right hand side borders of the tile $\tau$.
We note that, because of the connectedness of the tiles,
there exists for every $\beta^-(\tau) < \beta < \beta^+(\tau)$
some $\alpha > 0$ such that $(\beta, \alpha) \in \tau$.
In the following Lemma, we will show that the
set of these parameters $\alpha$ is actually an interval
(see also \cite[Proof of Lemma 6]{tibshirani2013lasso}).

\begin{lemma}
  Assume that $\tau \in \mathcal{T}$ and
  $\beta^-(\tau) < \beta < \beta^+(\tau)$.
  Then the set of parameters $\alpha>0$ with
  $(\beta,\alpha) \in \tau$ is a non-empty interval.
\end{lemma}

\begin{proof}
  Assume that $\alpha_0 < \alpha_1$ are such that $(\beta,\alpha_i)
  \in \tau$ for $i = 0$, $1$. 
  Then $u_{\beta,\alpha_i}$ satisfy the conditions
  \[
  u_{\beta,\alpha_i,I_{\tau}} =
  (B_{\beta,I_\tau}^TB_{\beta,I_\tau})^{-1}(B_{\beta,I_\tau}^Ty_\beta
  - \alpha_i \sigma_\tau) 
  \]
  for $i = 0$, $1$, and
  \[
  | B_{\beta,j}^T(B_\beta u_{\beta,\alpha_i} - y_\beta)| \le \alpha_i.
  \]
  As a consequence, for every $0 < \lambda < 1$, we have that
  $u^{\lambda}:=\lambda u_{\beta,\alpha_1} + (1-\lambda)u_{\beta,\alpha_0}$
  satisfies the same conditions with $\alpha_i$ replaced by
  $\alpha_\lambda = \lambda \alpha_0 + (1-\lambda)\alpha_1$. Since
  $\supp u^{\lambda} = \supp u_{\beta,\alpha_i} = I_\tau$ and also the
  sign patterns on the support are the same, this implies that
  $u^{\lambda} = u_{\beta,\alpha_\lambda}$ and thus
  $(\beta,\alpha_\lambda) \in \tau$.
\end{proof}

Next we define for every tile $\tau \in \mathcal{T}$
and every $\beta^-(\tau) < \beta < \beta^+(\tau)$ the functions
\[
\begin{aligned}
\alpha^+_\tau(\beta) &= \sup\{\alpha > 0: (\beta,\alpha) \in \tau\},\\
\alpha^-_\tau(\beta) &= \inf\{\alpha > 0: (\beta,\alpha) \in \tau\},\\
\end{aligned}
\]
Provided these functions are continuous, a criterion
for which will be given in the next lemma,
their graphs form precisely the boundary of the tile $\tau$.
We note that the conditions in the next lemma are satisfied for almost all matrices $A$.

\begin{lemma}
\label{lem:continuous}
  Let $1 \le s \le m$.
  Assume that there exist no parameter $\beta > 0$,
  no subset $\emptyset \neq I \subset \{1,\ldots,n\}$ with $\lvert I \rvert \le s+1$,
  and no $j \in I$ such that the equations
  \[
  \bigl((B_{\beta,I}^T B_{\beta,I})^{-1}B_{\beta,I}^T y_\beta\bigr)_j = 0
  \]
  and
  \[
  \bigl((B_{\beta,I}^TB_{\beta,I})^{-1}\sigma\bigr)_j = 0
  \]
  are simultaneously satisfied.
  Then for all tiles $\tau$ with $\lvert I_\tau \rvert \le s$,
  the functions $\alpha^+_\tau$ and $\alpha^-_\tau$ are continuous.
\end{lemma}

\begin{proof}
  Assume to the contrary that, for some tile $\tau$ there exists
  $\beta^-(\tau) < \hat{\beta} < \beta^+(\tau)$ such that the
  function $\alpha^-_\tau$ is discontinuous at $\hat{\beta}$.
  Moreover, denote
  \[
  \alpha_l = \liminf_{\beta \to \hat{\beta}} \alpha^-_\tau(\beta)
  \qquad\text{ and }\qquad
  \alpha_u = \limsup_{\beta \to \hat{\beta}} \alpha^-_\tau(\beta).
  \]
  Then there exists an index $j \not \in I_\tau$ and $\sigma_j \in \{\pm 1\}$,
  as well as a sequence $\beta_k \to \hat{\beta}$ such that
  \[
  \sign(u_{\beta_k,\alpha_1,j}) = \sign(u_{\beta_k,\alpha_2,j}) = \sigma_j
  \]
  for all $k$ and some $\alpha_l \le \alpha_1 < \alpha_2 \le \alpha_u$.
  As a consequence, the KKT conditions~\eqref{eq:KKT}
  imply that
  \[
  B_{\beta_k,j}^T(B_{\beta_k}u_{\beta_k,\alpha_i}-y_{\beta_k}) = -\alpha_i \sigma_j
  \]
  for all $k \in \N$ and $i = 1,\,2$.
  Taking the limit $k \to \infty$ and recalling the continuous
  dependence of $u_{\beta,\alpha}$ on both $\beta$ 
  and $\alpha$ (see Lemma~\ref{le:cont}), it follows that also
  \begin{equation}\label{eq:alphaconth1}
    B_{\hat{\beta},j}^T (B_{\hat{\beta}}u_{\hat{\beta},\alpha_i}-y_{\hat{\beta}}) = -\alpha_i \sigma_j.
  \end{equation}
  On the other hand, the continuous dependence of $u_{\beta,\alpha}$ on both $\beta$
  and $\alpha$ also implies that $u_{\hat{\beta},\alpha,j} = 0$ for all
  $\alpha_l \le \alpha \le \alpha_u$.

  Let now $I = I_\tau \cup \{j\}$.
  Then $\supp(u_{\hat{\beta},\alpha}) \subset I$ for all $\alpha_l \le \alpha \le \alpha_u$,
  and thus, because of~\eqref{eq:alphaconth1}, we have
  \[
  u_{\hat{\beta},\alpha_i,I} = (B_{\hat{\beta},I}^T B_{\hat{\beta},I})^{-1}(B_{\hat{\beta},I}^T y_{\hat{\beta}}-\alpha_i \sigma_I)
  \]
  for all $\alpha_l \le \alpha \le \alpha_u$.
  Because $u_{\hat{\beta},\alpha,j} = 0$ for all these $\alpha$, it follows that both
  \[
  ((B_{\hat{\beta},I}^T B_{\hat{\beta},I})^{-1}\sigma_I)_j = 0
  \]
  and
  \[
  \bigl((B_{\hat{\beta},I}^T B_{\hat{\beta},I})^{-1}(B^T_{\hat{\beta},I}y_{\hat{\beta}}\bigr)_j = 0,
  \]
  which contradicts the assumption that these two terms cannot be simultaneously
  equal to zero.
\end{proof}

\paragraph{Graph structure of the tiling}
In the following we discuss the structure of the
tiling $\mathcal{T}$ in more detail in order to 
to formulate an algorithm for its computation.
To that end, we introduce the structure of a
directed multi-graph on $\mathcal{T}$ by including an edge
from a tile $\tau$ to another tile $\tau'$ for each maximal
open and non-empty interval $S \subset\R$ for which
\begin{equation}
\label{eq:shared_boundary}
\alpha^-_{\tau}(\beta) = \alpha^+_{\tau'}(\beta)
\qquad\text{ for all } \beta \in S.
\end{equation}
Two tiles are connected by an edge $e$ whenever
they share a common boundary, and we define the edge
to run from the tile with larger values of $\alpha$
to the tile with smaller values of $\alpha$.
Moreover, we denote
by $S_e$ the maximal interval for which~\eqref{eq:shared_boundary} holds.
\begin{figure}[t]
\centering
\begin{tikzpicture}[auto, inner sep = 0pt]
\begin{scope}
\node[inner sep=0pt, outer sep=0pt, anchor=south west] (image) at (-14,0)
{\includegraphics[width=\textwidth]{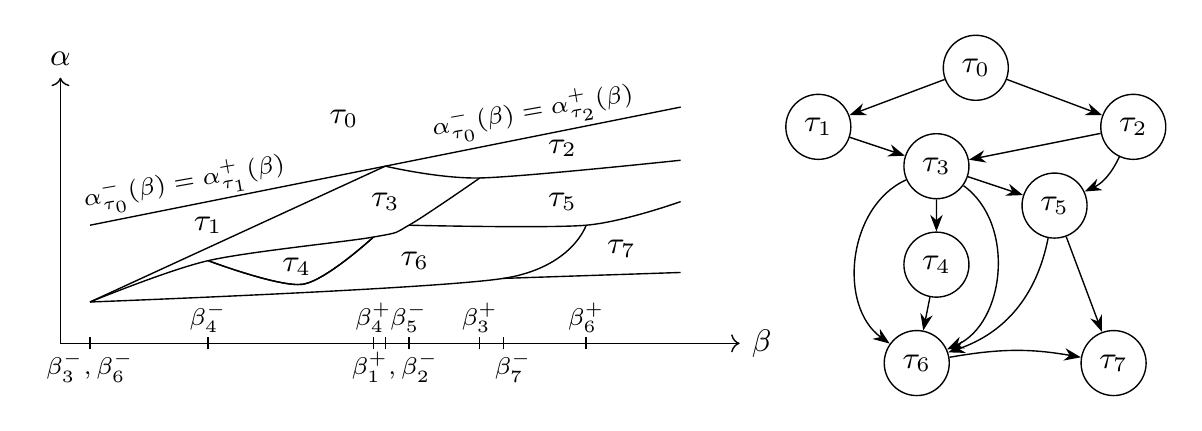}};
\end{scope}
\end{tikzpicture}
\caption{Example of part of a tiling (\emph{left}) and the corresponding directed multigraph (\emph{right}).
Note that the tile $\tau_3$ is connected to the tile $\tau_6$ via two different edges. The $\beta$ ticks on the $x$-axis are the left and the right borders of the tiles, i.e., $\beta_j^{\pm} = \beta^\pm(\tau_j)$. 
}
\label{fig:tiling_graph_explanation}
\end{figure}
We note here that it is possible that two tiles are connected
by several edges, if their common boundary consists of disjoint intervals, as depicted in 
 Figure \ref{fig:tiling_graph_explanation} (left panel).

Now let $\tau$ be any tile with edges $e_1$,\ldots, $e_k$ starting at $\tau$.
Then each edge $e_j$ corresponds to a maximal open subinterval $S_{e_j}$ of
$(\beta^-(\tau),\beta^+(\tau))$, and $S_{e_j} \cap S_{e_i} = \emptyset$
for all $i\neq j$. Therefore we can define an order on the set of edges
starting in $\tau$ by setting $e_j < e_i$ if $S_{e_j} < S_{e_i}$.
Also, we can order the edges leading into $\tau$ in analogous manner.

In order to find simultaneously the edges and the functions
$\alpha_\tau^\pm(\beta)$, it is possible to employ an adaptation of
the Lasso path algorithm to the multi-parameter problem as presented in Lemma \ref{le:lassopath}.

\begin{lemma}
Assume that $e$ is an edge leading from the tile $\tau$ into tile
$\tau'$, and that $S_e$ is the interval on which~\eqref{eq:shared_boundary}
holds. For $1 \le k\le n$ denote $\tilde{\alpha}_{k}(\tau,\beta) : = \tilde{\alpha}_{k}(I_\tau, \sigma_\tau,\beta)$ as defined in \eqref{eq:talpha} and let 
\begin{equation}\label{eq:K}
\mathcal{K}(\tau,\beta)
:= \{k: \tilde{\alpha}_k(\tau,\beta) < \alpha^+_\tau(\beta)\}.
\end{equation}
Then for each $\beta \in S_e$ and each $j \in I_{\tau} \Delta I_{\tau'},$
we have $j \in \mathcal{K}(\tau,\beta)$ 
and
\begin{equation}\label{eq:alpha_tau1}
\alpha^-_\tau(\beta)
= \alpha^+_{\tau'}(\beta)
=
\tilde{\alpha}_{j}(\tau,\beta) = \max_{k\in\mathcal{K}(\tau,\beta)} \tilde{\alpha}_{k}(\tau,\beta).
\end{equation}
If additionally $j \in I_{\tau'} \setminus I_\tau$, then $ \sigma_{\tau',j}$ is equal to the parameter $\gamma$ in the definition of $\tilde{\alpha}_{j}(I_\tau, \sigma_\tau,\beta)$, see \eqref{eq:talpha}.
\end{lemma}

\begin{remark}
  The previous lemma provides a method for simultaneously computing 
  the functions $\alpha^-_\tau(\beta)$ and finding the edges leading
  out of $\tau$: In the special case where for some given $\beta$
  the maximum in~\eqref{eq:alpha_tau1}
  is attained at a single index $k$,
  the neighboring tile $\tau'$ has the support $I_\tau \cup \{k\}$
  if $k \not\in I_\tau$ and $I_\tau\setminus\{k\}$ otherwise.
  Also, the boundary between $\tau$ and $\tau'$ is in a neighborhood
  of $\beta$ given by the function $\tilde{\alpha}_{k}(\tau,\cdot)$.
\end{remark}

In order to simplify further discussion, we introduce the notion of  parents
and children. Given a tile $\tau$, we denote by $\mathcal{P}(\tau)$
the set of its parents, that is, all tiles with edges leading to $\tau$,
and by $\mathcal{C}(\tau)$ its children, that is, all tiles with edges starting
from $\tau$.
Using the order of the edges between an edge $\tau$ and its children,
we can in particular define the youngest child $C^-_\tau$ and the oldest child $C^+_\tau$
in the following way:
We denote by $C_\tau^-$ the child of $\tau$ corresponding to the minimal edge
starting in $\tau$, and by $C_\tau^+$ the child corresponding to the maximal edge.
Because there might be multiple edges between $\tau$ and any of its children,
it can happen that $C_\tau^- = C_\tau^+$ even though the tile $\tau$ has multiple
children.
In an analogous manner, we define $P_\tau^-$ and $P_\tau^+$ to be the tile
corresponding to the minimal and maximal edge leading into $\tau$ and call
these tiles the youngest and the oldest parent of $\tau$, respectively.

Finally, we recall that for each parameter $\beta > 0$ and sufficiently large $\alpha$, we have
$u_{\beta,\alpha} = 0$.
As a consequence, the directed graph describing the tiling is actually
rooted, with the root given by the tile corresponding to
the zero solution with support set $I_\tau = \emptyset$.
We denote this root tile as $\tau_0$.
This is also the only tile that does not have any parents, and
all other tiles are descendants of $\tau_0$.  Figure \ref{fig:tiling_graph_explanation} depicts the directed graph (right panel) representing the tiling (left panel).

\section{Algorithmic approach}
\label{sec:algorithm}

In the following, we will discuss an algorithmic approach to
the construction of the complete support tiling, that is,
of all the tiles $\tau$ together with the corresponding
supports $I_\tau$ and sign patterns $\sigma_\tau$,
and the boundaries of $\tau$, which are given
by $\beta^\pm(\tau)$ and
the functions $\alpha^\pm\colon (\beta^-(\tau),\beta^+(\tau)) \to \R$.
Starting with the tile $\tau_0$, we construct the lower boundary
$\alpha^-(\tau)$ and the children $\mathcal{C}(\tau)$
of the current tile $\tau$ using equation~\eqref{eq:alpha_tau1}.
To that end, we will first subdivide the interval $(\beta^-(\tau),\beta^+(\tau))$
into subintervals on which the index set $\mathcal{K}(\tau,\beta)$
is constant. On each of these subintervals, the indices 
that are either added to or removed from the current support
can be identified as $\argmax_{k} \tilde{\alpha}_{k}(\tau,\beta)$.
In order to simplify the algorithm, we will assume that
these maxima are attained at a single index for almost all parameters $\beta$.
We note that this is not a severe restriction as it
holds for almost all matrices $A$.
Also, a similar restriction has been used in the original Lasso path algorithm \cite{efhajoti04}.

\begin{postulate}
  For each tile $\tau$ there are at most finitely many
  values of $\beta^-(\tau) < \beta < \beta^+(\tau)$ such
  that the maximum in~\eqref{eq:alpha_tau1} is attained
  simultaneously for different indices $j$.
\end{postulate}

Since the functions $\tilde{\alpha}_{k}$, which define the
maximum in~\eqref{eq:alpha_tau1}, are rational, this
is equivalent to the assumption that all these functions are
different. As a consequence, this assumption is very weak and can be safely assumed to hold in all practical situations.
\medskip

Each tile $\tau$ processed by the algorithm is defined
by means of its support $I_\tau$ and sign structure $\sigma_\tau$,
together with a range of $\beta$ values $(\beta^-(\tau),\beta^+(\tau))$
and all of its parents with edges defined over these $\beta$ values.
In particular, this means that its upper boundary $\alpha^+_\tau$
is well-defined for the given $\beta$ range.
However, we do not necessarily assume that all the children
of $\tau$ have already been constructed, and thus the lower
boundary $\alpha^-_\tau$ need not be well-defined everywhere.
In such a case, we say that the tile $\tau$ is incomplete.
See Figure~\ref{fig:illustration_non_exhaustive_parameters}
for a sketch of a situation where this occurs.

In each step of the algorithm, we choose an incomplete tile $\tau$.
Then we compute all of its children together with the function
$\alpha^-_\tau$ and thus complete the tile.
After its completion, we merge newly created children
with previously established tiles, if necessary.

\subsection{Computation of children}
\label{se:finding_successors}
\begin{figure}[t]
\centering
    \begin{subfigure}[b]{0.475\textwidth}
        \includegraphics[width=\textwidth]{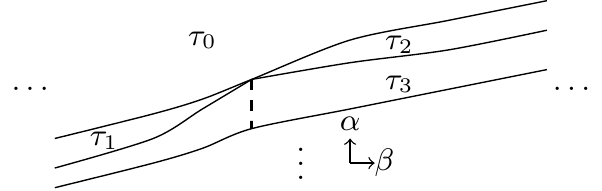}
        \caption{}
        \label{fig:illustration_non_exhaustive_parameters_1}
    \end{subfigure}
    \begin{subfigure}[b]{0.475\textwidth}
        \includegraphics[width=0.7\textwidth]{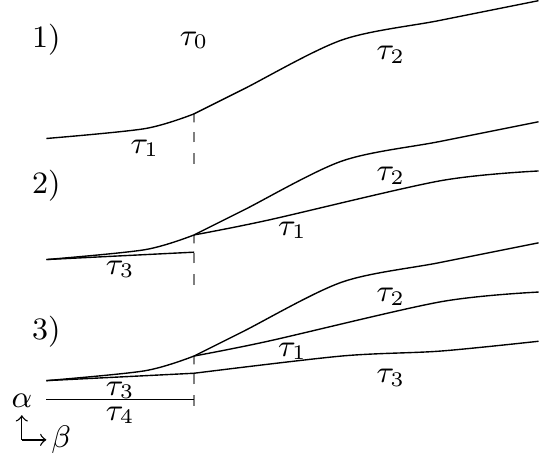}
        \caption{}
        \label{fig:illustration_non_exhaustive_parameters_2}
    \end{subfigure}
\caption{Exemplary extracts of support tilings where incomplete ranges upon the first discovery occur. (a): Coming from $\tau_1$ or $\tau_2$, the first discovery of $\tau_3$ would always only reveal the left (from $\tau_1$) or right (from $\tau_2)$) part of $(\beta^-(\tau_3),\beta^+(\tau_3))$. (b): In Step 1) we discover a part of $\tau_1$ and the complete $\tau_2$. In Step 2), we discover a part of $\tau_3$ from $\tau_1$ and subsequently the rest of $\tau_1$ through $\tau_2$. Hence we search for children of $\tau_1$ before having discovered it completely and this situation will repeat itself more often.}
\label{fig:illustration_non_exhaustive_parameters}
\end{figure}

Assume that we are given an incomplete tile $\tau$
with $\beta$-range $S_\tau := (\beta^-(\tau),\beta^+(\tau))$
and current (incomplete) set of children $\mathcal{C}(\tau)$.
We denote by $\mathcal{E}(\tau)$ the set of edges
connecting $\tau$ with its children.
Moreover, we denote by 
\[
T_\tau := \bigcup_{e\in\mathcal{E}(\tau)} \overline{S}_e
\]
the subset of $S_\tau$ where $\alpha^-(\tau)$ is known.

We next compute a subdivision of $S_\tau\setminus T_\tau$ into subintervals
on each of which the index set $\mathcal{K}(\tau,\beta)$ is constant, see~\eqref{eq:K}.
Given a connected component $L$ of $S_\tau\setminus T_\tau$,
any of the functions $\tilde{\alpha}_{k}(\tau,\cdot)$
either is below $\alpha^+_\tau$ on the whole interval $L$,
above $\alpha^+_\tau$ on the whole interval $L$,
or has discontinuities within $L$.
In the latter case, because of the definition of the functions
$\tilde{\alpha}_k(\tau,\cdot),$ see~\eqref{eq:talpha},
these discontinuities can only occur at the parameters $\beta$ for
which either 
\begin{equation}\label{eq:subdiv}
\bigl((B_{\beta,I_\tau}^T B_{\beta,I_\tau})^{-1}\sigma_\tau\bigr)_k = 0
\qquad\text{ if } k\in I_\tau,
\end{equation}
or
\[
B_{\beta,k}^TB_{\beta,I_\tau}(B_{\beta,I_\tau}^T B_{\beta,I_\tau})^{-1}\sigma_\tau=\gamma
\qquad\text{ if } k\not\in I_\tau.
\]
In the following result, we will show that, actually,
it is only the first of these cases that can occur.

\begin{lemma}\label{lemma:indices_new}
  For all $k\not\in I_\tau$ we have
  \[
  B_{\beta,k}^TB_{\beta,I_\tau}(B_{\beta,I_\tau}^T B_{\beta,I_\tau})^{-1}\sigma_\tau \neq 
  \gamma,
  \]
  where $\gamma \in \{\pm 1\}$ is defined as in~\eqref{eq:talpha}.
\end{lemma}

\begin{proof}
  The optimality conditions~\eqref{eq:KKT} imply that
  \begin{equation}\label{eq:indices_1}
  \lvert B_{\beta,k}^T(y_\beta - B_{\beta,I_\tau}u_{\beta,\alpha,I_\tau})\rvert \le \alpha
  \end{equation}
  for all $\alpha^-(\beta) \le \alpha \le \alpha^+(\beta)$.
  Moreover,
  \[
  u_{\beta,\alpha,I_\tau} = (B_{\beta,I_\tau}^T B_{\beta,I_\tau})^{-1}(B_\beta^T y_\beta - \alpha \sigma_\tau).
  \]
  Inserting this into~\eqref{eq:indices_1}, we obtain the inequality
  \begin{equation}\label{eq:indices_2}
  \lvert B_{\beta,k}^T(\Id-B_{\beta,I_\tau}(B_{\beta,I_\tau}^TB_{\beta,I_\tau})^{-1}B_{\beta,I_\tau}^T)y_{\beta} + \alpha B_{\beta,k}^TB_{\beta,I_\tau}(B_{\beta,I_\tau}^TB_{\beta,I_\tau})^{-1}\sigma_\tau\rvert
  \le \alpha
  \end{equation}
  for all $\alpha^-(\beta) \le \alpha \le \alpha^+(\beta)$.

  Assume now that $k$ was not contained in the support of $u_{\beta,\alpha}$ in the
  previous step of the Lasso path algorithm. Then, see Lemma~\ref{le:lassopath},
  \[
  \gamma = \sign(B_{\beta,k}^T(\Id - B_{\beta,I_\tau}(B_{\beta,I_\tau}^TB_{\beta,I_\tau})^{-1}B_{\beta,I_\tau}^T)y_{\beta}).
  \]
  If additionally $B_{\beta,k}^TB_{\beta,I_\tau}(B_{\beta,I_\tau}^TB_{\beta,I_\tau})^{-1}\sigma_\tau = \gamma$,
  then~\eqref{eq:indices_2} reduces to an inequality of the form
  \[
  \lvert r + \alpha \sign(r) \rvert \le \alpha
  \]
  with $r \neq 0$, which is impossible.

  Conversely, if $k$ was contained in the support of $u_{\beta,\alpha}$ in the previous
  step of the Lasso path algorithm, then the upper
  boundary of the tile $\tau$ is at the point $\beta$ given by
  \[
  \alpha^+(\beta) = \frac{B_{\beta,k}^T(\Id - B_{\beta,I_\tau}(B_{\beta,I_\tau}^TB_{\beta,I_\tau})^{-1}B_{\beta,I_\tau}^T)y_{\beta}}{-\gamma - B_{\beta,k}^TB_{\beta,I_\tau}(B_{\beta,I_\tau}^TB_{\beta,I_\tau})^{-1}\sigma_\tau}.
  \]
  Thus
  \[
  B_{\beta,k}^T(\Id - B_{\beta,I_\tau}(B_{\beta,I_\tau}^TB_{\beta,I_\tau})^{-1}B_{\beta,I_\tau}^T)y_{\beta}
  = -\alpha^+(\beta)(\gamma+B_{\beta,k}^TB_{\beta,I_\tau}(B_{\beta,I_\tau}^TB_{\beta,I_\tau})^{-1}\sigma_\tau).
  \]
  If in this case $B_{\beta,k}^TB_{\beta,I_\tau}(B_{\beta,I_\tau}^TB_{\beta,I_\tau})^{-1}\sigma_\tau = \gamma$,
  then we obtain from~\eqref{eq:indices_2} the inequality
  \[
  \lvert -2\gamma \alpha^+(\beta) + \alpha \gamma \rvert \le \alpha
  \]
  for all $\alpha^-(\beta) \le \alpha \le \alpha^+(\beta)$,
  which is again impossible.
\end{proof}

As a consequence of Lemma~\ref{lemma:indices_new}, if we subdivide
the set $S_{\tau}\setminus T_\tau$ into intervals
\begin{equation}\label{eq:Ttau}
S_{\tau}\setminus T_\tau = \mathop{\dot{\bigcup}}_{k=1}^K (\beta_k^-,\beta_k^+),
\end{equation}
where the boundary points $\beta_k^\pm$ of these intervals
are either original boundary points of $S_\tau$ or points
$\beta$ for which $\bigl((B_{\beta,I_\tau}^T B_{\beta,I_\tau})^{-1}\sigma_\tau\bigr)_i = 0$
for some $i\in I_\tau$, then the set $\mathcal{K}(\tau,\beta)$
remains constant over each of the subintervals $(\beta_k^-,\beta_k^+)$.

Choose now one of these subintervals $(\beta_k^-,\beta_k^+)$
and let $\mathcal{K}_k := \mathcal{K}(\tau,\beta)$ for any $\beta \in (\beta_k^-,\beta_k^+)$.
Then we can compute a preliminary set of all children of
$\tau$ within the interval $(\beta_k^-,\beta_k^+)$ as follows:
\begin{algorithmic}
	\STATE 1. Start with $\beta = \beta_k^{-}$.
	\STATE 2. Compute the index $j$ for which
  \[
  j = \argmax\{ \tilde{\alpha}_{i}(\tau,\beta+\epsilon): i \in \mathcal{K}_k\}
  \]
  for some small $\epsilon>0$.
  \STATE 3. Add a child $\tilde{\tau}$ with $\beta^-(\tilde{\tau}) = \beta$ to $\tau$:
  \bindent
\IF{$j \in \tau$}
        \STATE Let $I_{\tilde{\tau}} = I_\tau \setminus\{j\}$ and
  $\sigma_{\tilde{\tau},i} = \sigma_{\tau,i}$ for $i \in I_\tau$, $i\neq j$.
    \ELSE
        \STATE Let $I_{\tilde{\tau}} = I_\tau \cup\{j\}$, $\sigma_{\tau,j} = \gamma$,
  and $\sigma_{\tilde{\tau},i} = \sigma_{\tau,i}$ for $i \in I_\tau$.
    \ENDIF
    \eindent
    \STATE 4. Compute the next solution $\tilde{\beta} > \beta$
    of any of the equations $\tilde{\alpha}_i(\tau,\beta) = \tilde{\alpha}_j(\tau,\beta)$,
    $i\in\mathcal{K}_k$:
  \bindent
  \IF{$\tilde{\beta} < \beta_k^+$}
  		\STATE Set $\beta^+(\tilde{\tau}) = \tilde{\beta}$ and repeat from 2 with $\beta = \tilde{\beta}$.
  \ELSE
 		\STATE Set $\beta^+(\tilde{\tau}) = \beta_k^+$ and stop.
   \ENDIF  
   \eindent
\end{algorithmic}

Now assume that all subintervals $(\beta_k^-,\beta_k^+)$ have
been processed in that manner.
Then the tile $\tau$ is completed in the sense that 
its lower boundary $\alpha_\tau^-$ is defined everywhere
and that a preliminary set of children of $\tau$ is defined on the whole
range of values of $\beta$ in $S_\tau$.
However, it is possible that some of these preliminary children
actually should be merged together, because they point to same tile.
This will occur at the boundaries of the subintervals
$(\beta^-_k,\beta_k^+)$, as all of these subintervals
have been processed independently from each other.

In order to merge children, we may scan through
all preliminary children of $\tau$ from oldest to youngest.
If we then find two adjacent children with same support $\tilde{\tau}$
and sign pattern $\tilde{\sigma}$, we merge them in the sense
that we treat them as only one child. The corresponding $\beta$ range for this new child is the union of the ranges of the preliminary children.

\subsection{Merging procedure}
\label{se:algorithm_merging}

After a tile $\hat{\tau}$ has been processed as described
in Section~\ref{se:finding_successors}, children of this tile
will be defined on the whole interval $(\beta^-(\hat{\tau}),\beta^+(\hat{\tau}))$.
However, the children defined on the boundaries of this interval,
that is, the oldest and the youngest child $C^\pm_{\hat{\tau}}$ of $\hat{\tau}$,
might coincide with children from neighboring tiles to the left or right,
which again might require the merging of these children.
In order to perform this merging, we make use of the ordered graph
structure we imposed on the tiling.

We first consider the youngest child $C^-_{\hat{\tau}}$. Starting with
$\hat{\tau}$, we trace back the line of youngest parents, until we
reach an ancestor $\tilde{\tau}$ for which the branch from
$\tilde{\tau}$ to $C^-_{\hat{\tau}}$ does not start with the youngest
child of $\tilde{\tau}$. In case no such tile exists, the tile
$C^-_{\hat{\tau}}$ has no possible merging partners. 
Now denote by $\tau'$ the child of $\tilde{\tau}$ in the line leading
to $C^-_{\hat{\tau}}$, and let $\tau''$ be its next younger
sibling. Then the potential merging partners for $C^-_{\hat{\tau}}$
are precisely the tiles in the line of oldest children of
$\tau''$. We therefore follow this line until we find a tile that can
be merged with $C^-_{\hat{\tau}}$. Again, in case no such tile exists,
the tile $C^-_{\hat{\tau}}$ has no potential merging partners, see Figure~\ref{code:find_left_candidate}.

For the oldest child $C^+_{\hat{\tau}}$ the approach is
similar. However, here we trace back the line of oldest parents until
we find an ancestor for which the branch to $C^-_{\hat{\tau}}$ does
not start with its oldest child. Then we follow the line of youngest
children of this tile's next older sibling in order to find
possible merging partners of $C^+_{\hat{\tau}}$.

Note that the result of this merging procedure will always be an
incomplete tile, even if the tile we merge $C^\pm_{\hat{\tau}}$ with has
already been processed before.
Also, it is possible that both merging operations have
to be performed in case the tile $\hat \tau$
has only a single child $C^+_{\hat{\tau}} = C^-_{\hat{\tau}}$,
see Figure~\ref{fig:full_merging_procedure}.

\begin{figure}[t]
\centering
\begin{tikzpicture}[auto, inner sep = 0pt]
\begin{scope}
\node[inner sep=0pt, outer sep=0pt, anchor=south west] (image) at (-14,0)
{\includegraphics[width=\textwidth]{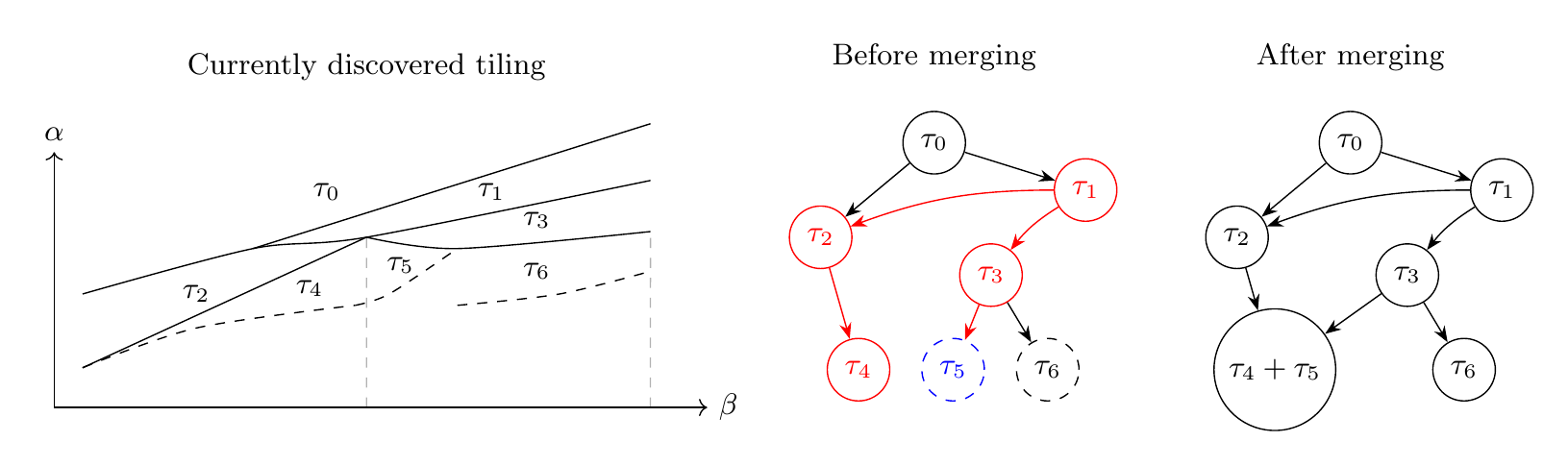}};
\end{scope}
\end{tikzpicture}
\caption{The merging procedure is performed by using graph structure, imposed on the tiling (left), and analysing the youngest child of tile $\tau_3$ as indicated on the right.}
\label{code:find_left_candidate}
\end{figure}

\begin{figure}[t]
\centering
\begin{tikzpicture}[auto, inner sep = 0pt]
\begin{scope}
\node[inner sep=0pt, outer sep=0pt, anchor=south west] (image) at (-14,0)
{\includegraphics[width=\textwidth]{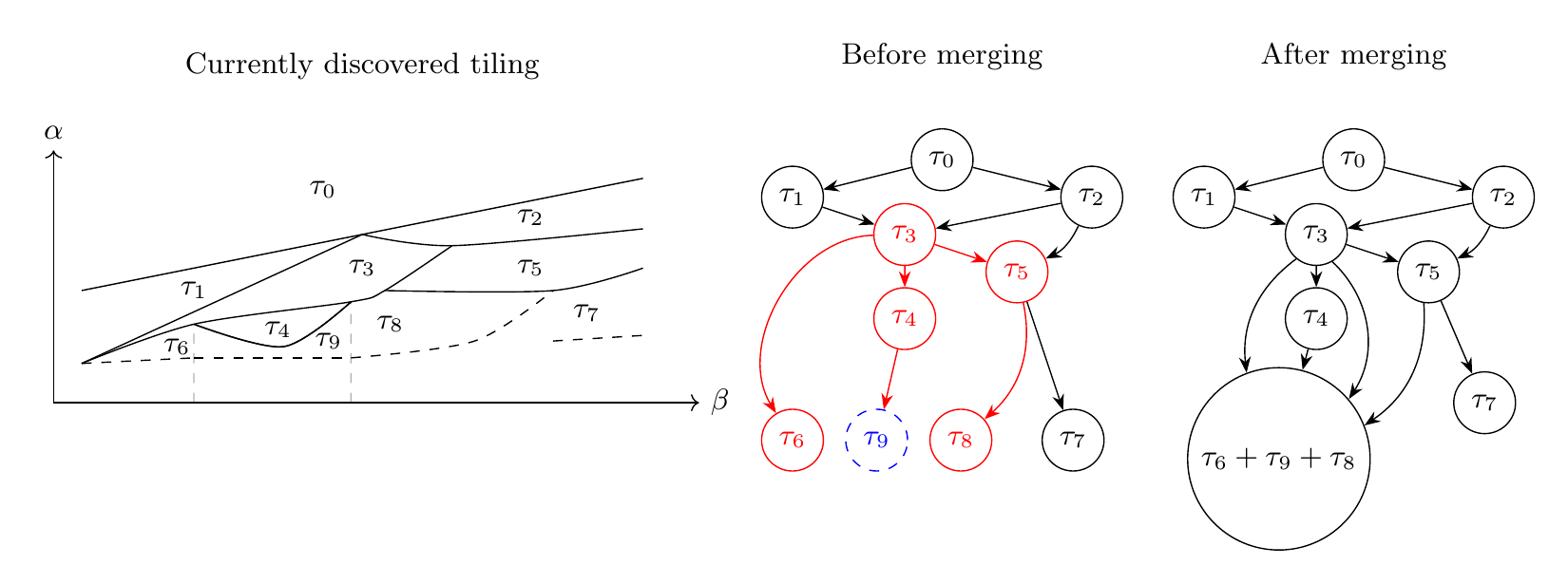}};
\end{scope}
\end{tikzpicture}
\caption{Despite the processed tile $\tau_4$ having only a single child after the computation of the children, potential merging can happen in both directions as indicated on the right.}
\label{fig:full_merging_procedure}
\end{figure}

\subsection{Complete algorithm}

Fix an interval $(\beta_{\min}, \beta_{\max})$ and the upper bound for the support size $s$. One obtains the tiles containing solutions with sparsity level up to $s$ within the given $\beta-$range as follows:
\begin{algorithmic}
\STATE  Initiate $\hat{\mathcal{T}}_a = \left\{\hat{\tau}_0 : I_{\hat{\tau}_0} = \emptyset,\ \sigma_{\hat{\tau}_0} = 0,\  S_{\hat{\tau}_0} = (\beta_{\textrm{min}},\beta_{\textrm{max}})\right\}$.
  \WHILE{$\hat{\mathcal{T}}_a$ contains an uncompleted $\hat{\tau}$ with $|I_{\hat{\tau}}| < s$}
  		\STATE Pick an uncompleted $\hat{\tau}$ with $|I_{\hat{\tau}}| < s$.
  		\STATE $R_{\hat{\tau}} \gets \textit{findChildren}(A, y, \hat{\tau})$.
 		\STATE Tentatively assign all $\tau' \in R_{\hat{\tau}}$ as children of $\hat{\tau}$.
		\STATE $\hat{\mathcal{T}}_a \gets \textit{mergeChildren}(R_{\hat{\tau}})$.
   \ENDWHILE
\end{algorithmic}

In Sections \ref{se:finding_successors} and \ref{se:algorithm_merging}
above we have discussed the general framework of the proposed
algorithm, each step of which consists in the completion of a tile,
the construction of its children, and then the possible merger of the
newly created children with existing tiles. However, we have completely left open
the question of how to choose the next tile to be processed.

In practical applications, we may assume that we are given a range of
values $(\beta_{\min},\beta_{\max})$ as well as an upper bound $s$ for the
size of the support of the vector we want to reconstruct. Thus our
goal is the construction of all tiles of support size up to $s$ with
$\beta$-ranges intersecting the interval
$(\beta_{\min},\beta_{\max})$. In order to do this efficiently, it is
natural to compute first the tiles with the smallest support. That
is, we always choose the next tile to be processed amongst those
incomplete tiles $\tau$ where $\# I_\tau$ is the smallest.

If there are several different incomplete tiles with the same support
size, we propose to process tiles that already have children before
other tiles. 
Amongst these tiles we first process those with the smallest lower
bound $\beta^-(\tau)$ of their $\beta$-ranges.
That is, we effectively process the tiles from youngest to oldest
tile, although the opposite order would equally make sense.
In the possible, though highly unlikely, case when there exist several
tiles with the same support size and same lower $\beta$ bound, we process the oldest of them first, that is, the one with the
largest value of $\alpha_\tau^-(\beta^-(\tau))$.

\subsection{LARS algorithm}

As an alternative to the Lasso path algorithm, it is also possible to
compute the possible supports of the solution using the LARS
algorithm, although there is no immediate connection to the proposed
multi-penalty functional \eqref{eq:multi}.
The LARS algorithm differs from the Lasso path algorithm by
disallowing entries that have joined the support to leave it again
when the parameter $\alpha$ is decreased. This means that for the
computation of $\alpha_\tau^-$ in~\eqref{eq:alpha_tau1} the maximum is
only taken over indices $k$ not contained in $I_\tau$. This simplifies
the computations somewhat, as the number of functions used in the
computation of $\alpha_\tau^-$ decreases.
Even more, the creation of a subdivision of the $\beta$ range of
the currently processed tile and thus the solution of the
equations~\eqref{eq:subdiv} is not necessary for the LARS algorithm.

Another advantage of using this alternative method is that it results
in the graph of the tiling consisting of different layers for the
different support sizes. If we define
\begin{equation*}
L_k = \left\{ \tau \in \mathcal{T}: \# I_{\tau} = k  \right\}
\end{equation*}
as the set of tiles with support size $k$, then the parents of tiles in
the layer $L_k$ are all contained in the layer $L_{k-1}$ and vice
versa. In particular, if we want to compute all the supports up to
support size $s$, this can be easily achieved by constructing these
different layers one at a time until we reach $L_s$.
In contrast, if the Lasso path algorithm is used, the resulting graph
cannot be expected to have the same layered structure, and it might be
necessary to also compute tiles with support larger than $s$.
In addition, the merging of children becomes much simpler with the
LARS algorithm, as the only potential merging partners of a tile are
its immediate older and younger sibling, respectively.

Finally we note that it has been reported numerous times in the literature
\cite{efhajoti04} that the Lasso path and LARS algorithm differ barely
for high-dimensional problems. We observed the same behaviour in our numerical experiments.\footnote{Follow the links to repositories in Section \ref{se:numerical_experiments} to see results on the LARS variant.} Therefore, the LARS variant of the Lasso-path algorithm can be considered as a useful alternative due to improved efficiency, also in the multi-penalty framework.

\section{Numerical Realization}
\label{se:heuristics}

The computation of the tiling based on the outlined method raises two main numerical
difficulties. First, for the computation of the subdivision of the
$\beta$-range $S_\tau$ of the tiling, it is
necessary to find the parameters $\beta$ for which
$((B_{\beta,I_\tau}^TB_{\beta,I_\tau})^{-1}\sigma_\tau)_i = 0$ for
some $i \in I_\tau$. Secondly, we have to solve equations of the form
$\tilde{\alpha}_{j,\gamma}(\tau,\beta) =
\tilde{\alpha}_{i,\delta}(\tau,\beta)$ for $\beta$. In the following,
we will describe numerical methods for the solutions of these problems
using heuristics concerning the behavior of the functions involved.

\subsection{Computation of the subdivision of $S_\tau\setminus T_\tau$}
\label{se:subdivision}

We assume now that we are given an incomplete tile $\tau$ and a subset
$S_\tau\setminus T_\tau$ of
$S_\tau = (\beta^-(\tau),\beta^+(\tau))$ on which we want to determine the
children of $\tau$.
In order to simplify the further notation, we assume that this subset
is the whole interval $(\beta^-(\tau),\beta^+(\tau))$. According to the argumentation in
Section~\ref{se:finding_successors}, it is then necessary to find the values of $\beta$ for
which
\begin{equation}\label{eq:contreq}
s_{\tau,i}(\beta) := ((B_{\beta,I_\tau}^TB_{\beta,I_\tau})^{-1}\sigma_\tau)_i = 0
\end{equation}
for some $i \in I_\tau$. In other words, we have to find the
parameters $\beta$ for which the function $s_{\tau,i}$ changes
sign. In order to find these points efficiently, we make the following
assumption:

\begin{postulate}\label{post:contreq}
  For any given tile $\tau$ and any $i \in I_\tau$, the
  equation~\eqref{eq:contreq} holds for at most one parameter $\beta
  \in (\beta^-(\tau),\beta^+(\tau))$.
\end{postulate}

If Assumption~\ref{post:contreq} holds, it is possible to
compute the points where~\eqref{eq:contreq} is satisfied by a simple
bisection method. To that end, we compute first for each index $i \in I_\tau$
the values $s_{\tau,i}(\beta^\pm(\tau))$. In case the signs of
$s_{\tau,i}(\beta^\pm(\tau))$ are different, there is some (unique) point
$\beta^-(\tau) < \beta_i < \beta^+(\tau)$ for which $s_{\tau,i}(\beta_i) =
0$. This point can be found by a bisection method.

Note that $\alpha s_{\tau,i}(\beta)$ is precisely the amount of
shrinkage that is applied to $u_{\beta,\alpha,i}$ because of the Lasso
regularization. Since this amount should usually increase with
$\beta$, we would expect that the functions $s_{\tau,i}(\beta)$ are
usually monotoneous in $\beta$. Assumption~\ref{post:contreq} is even
weaker than this monotonicity, and thus we regard it as rather mild.

\subsection{Computation of children}

Assume that we are given a tile $\tau$ and a subinterval $(\beta^-,\beta^+)$
of its $\beta$-range. 
Following the method in Section~\ref{se:finding_successors},
we have to find all the parameters $\tilde{\beta} \in [\beta^-,\beta^+]$,
where the index 
\[
j_{\tau}(\beta) := \argmax_{j \in \mathcal{K}(\tau,\beta)} \tilde{\alpha}_{j}(\tau,\beta)
\]
changes. Moreover, we can assume that the set $\mathcal{K} := \mathcal{K}(\tau,\beta)$
is independent of $\beta \in [\beta^-,\beta^+]$ and that the functions $\tilde{\alpha}_j(\tau,\beta)$,
$j \in \mathcal{K}$ are continuous on $[\beta^-,\beta^+]$.

In order to simplify the computations, we will make the following assumption:

\begin{postulate}
\label{post:heuristics}
For every index $k$, the sets $\{\beta \in [\beta^-,\beta^+] : k \in j_\tau(\beta)\}$
are either empty or intervals.
\end{postulate}

This assumption ensures that any specific index $j$ can only contribute 
to the maximizing envelope of the functions $\tilde{\alpha}_j(\tau,\cdot)$
on a connected set. As a consequence, if we are given $\beta^{(\ell)} < \beta^{(r)} \in [\beta^-,\beta^+]$
such that $j \in j_{\tau}(\beta^{(\ell)})$ and $j \in j_{\tau}(\beta^{(r)})$,
then we can immediately conclude that $j = j_{\tau}(\beta)$ on the whole interval
$(\beta^{(\ell)},\beta^{(r)})$.
This leads to the following divide-and-conquer approach:

For computing $j_\tau(\beta)$ on an interval
$[\beta^{(\ell)},\beta^{(r)}] \subset [\beta^-,\beta^+]$,
we compute first for all $j \in \mathcal{K}$ the values
\[
\tilde{\alpha}^{(l)}_j := \tilde{\alpha}_j(\tau,\beta^{(l)})
\qquad\text{ and }\qquad
\tilde{\alpha}^{(r)}_j := \tilde{\alpha}_j(\tau,\beta^{(r)}),
\]
and then set
\[
j^{(l)} := \argmax_{j\in\mathcal{K}} \tilde{\alpha}^{(l)}_j
\qquad\text{ and }\qquad
j^{(r)} := \argmax_{j\in\mathcal{K}} \tilde{\alpha}^{(r)}_j
\]
In case the index $j$ for which one of the maxima is attained is not unique,
we choose some sufficiently small $\epsilon > 0$ ($\epsilon \ll (\beta^{(r)}-\beta^{(l)})$)
and repeat the calculation at the point $\beta^{(l)} + \epsilon$ or $\beta^{(r)}-\epsilon$, respectively.

If $j^{(l)} = j^{(r)}$, then, according to Postulate~\ref{post:heuristics}, the index
$j_{\tau}(\beta)$ equals $j^{(l)}$ on the whole interval $(\beta^{(l)},\beta^{(r)})$,
and the computation is finished.
Else, we choose some $\beta^{(m)}$ with $\beta^{(l)} < \beta^{(m)} < \beta^{(r)}$,
and separately compute $j_{\tau}(\beta)$ with the same procedure
on the two intervals $[\beta^{(l)},\beta^{(m)}]$ and $[\beta^{(m)},\beta^{(r)}]$.

In order to define $\beta^{(m)}$, we employ the following strategy:
Usually, we set $\beta^{(m)} := (\beta^{(l)}+\beta^{(r)})/2$, that is,
we use a simple bisection method.
If, however, the index $j^{(r)}$ is such that $\tilde{\alpha}_{j^{(r)}}(\tau,\beta^{(l)})$
attains the second largest value within $\mathcal{K}$ and simultaneously
the index $j^{(l)}$ is such that $\tilde{\alpha}_{j^{(l)}}(\tau,\beta^{(r)})$ attains
the second largest value within $\mathcal{K}$ as well,
then we set $\beta^{(m)}$ to be any solution of the equation
\[
\tilde{\alpha}_{j^{(l)}}(\tau,\beta) = \tilde{\alpha}_{j^{(r)}}(\tau,\beta)
\qquad\text{ with } \beta \in (\beta^{(l)},\beta^{(r)}).
\]
In order to find such a solution, we use the secant method, modified in
such a way that the iterates stay within the interval $(\beta^{(l)},\beta^{(r)})$,
in which we can guarantee the existence of a solution.

\subsection{Computational complexity}
\label{sec:complexity}

In the following we will briefly discuss the computational complexity
of the proposed method. Here we assume that $m \le n$ (typically we
assume that $m$ is significantly smaller than $n$) and that we are only interested
in finding tiles for which the corresponding solution has a sparsity level
of at most $s_{\max} \ll m$.
Moreover, we note that the cost of the potential merging of tiles as described
in Section~\ref{se:algorithm_merging} is negligible compared to the cost
of actually finding the children and thus will be ignored in the further discussion.

To find the children of a specific tiling element $\hat{\tau}$, we need to calculate 
matrices $B_\beta$ respectively $B_{\beta, I}$ for some index set $I$ several times 
for numerous values of $\beta$. If we initially calculate the singular value decomposition of 
$A A^T \in \mathbb{R}^{m \times m}$, the full matrix $B_{\beta}$ can be computed afterwards 
in $\mathcal{O}(m^2 n)$ operations. Sub-sampled versions, e.g $B_{\beta, I}$, can be computed 
in $\mathcal{O}(m^2 \lvert I\rvert)$ steps. The initial singular value decomposition has to be performed once 
and has a cost of $\mathcal{O}(m^2n)$ as well.

Assume now that we want to compute the children of a given single tile $\hat{\tau}$
with support size $s = \lvert I_\tau \rvert$. Then we have to perform the following
procedures:

\begin{itemize}
\item For the computation of the subdivision $S_\tau$, it is necessary to
  solve the equation $s_{\tau,i}(\beta) = 0$ for each $i \in I_\tau,$ see~\eqref{eq:contreq}.
  Each evaluation of this function requires the computation of the matrix
  $B_{\beta,I_\tau}^TB_{\beta,I_\tau}$, which takes $\mathcal{O}(m^2s)$ operations,
  and then the solution of a linear system in $s$ variables, which we may assume takes
  a constant number of iterations up to a given accuracy.
  Since at most $s$ such equations have to be solved, the total number
  of iterations amounts to $\mathcal{O}(m^2 s^2)$.
  Note, however, that this is a worst case scenario, as this calculation
  is only necessary, if the signs of the function $s_{\tau,i}(\beta)$
  on the boundaries of the $\beta$-range of the tile $\tau$ differ.
  In practice, this occurs only rarely, and thus the computational costs of
  this step are much smaller.
  Moreover, for the LARS variation, the computation of $S_\tau$ is not performed at all.
\item For the actual computation of the children of the tile $\tau$,
  it is necessary to evaluate the functions $\tilde{\alpha}_j(\tau,\beta)$
  for different parameters $\beta$.
  The evaluation of a single such function takes $\mathcal{O}(m^2 s)$ operations,
  while the simultaneous evaluation of all these functions can be performed
  in $\mathcal{O}(m^2n)$ steps.
  Again, we can assume that we need a constant number of iterations in
  order to find the $\beta$ range for a cild for a given
  precision. Thus the total cost of this step will be
  $\mathcal{O}(m^2n)$ times the number of preliminary children that
  are produced.
\end{itemize}

In total, the number of operations for a given tile is of order
$\mathcal{O}(m^2 n)$ times the number of preliminary children that are
found. For the total cost, this has then to be multiplied with the
number of tiles that are processed. The latter is strongly dependent
on $s$, but also on the type of the measurement matrix. 

Compared to the original Lasso-path algorithm, which has a complexity
of $\mathcal{O}(smn)$, and also pLasso with a numerical complexity of
$\mathcal{O}(m^2n)$, the proposed method is therefore
more expensive. However, as the numerical experiments in Section \ref{se:numerical_experiments} indicate, the method leads to the improved accuracy and recovery rates. Therefore,  
the increased computational effort could be worthwhile. 

\section{Numerical experiments}
\label{se:numerical_experiments}
In this Section we provide extensive numerical experiments\footnote{Jupyter notebooks to the conducted experiments can be found at \url{https://github.com/soply/mp_paper_experiments}. The source code for conducting the experiments can be found in the repositories \url{https://github.com/soply/sparse_encoder_testsuite} and \url{https://github.com/soply/mpgraph}.}  to illustrate the effectiveness and robustness of our approach, compared to its single-penalty counterparts and the discretized multi-penalty approach.

In our experiments, we consider the model problem
\begin{equation*}
A(u^\dagger+v) + \delta = y,
\end{equation*}
where $A \in \R^{m \times n}$ is a linear measurement matrix, $u^\dag$ is a sparse vector, $v^\dag$ is a signal noise vector, and $\delta$ is a measurement noise vector. We consider three types of random measurement matrices, corresponding to different compressed sensing settings: Gaussian random matrices, partial random circulant matrices \cite{gray2006toeplitz}, and Gamma/Gaussian matrices \cite{jiaro15}. For each matrix type and each configuration as detailed below, we run 100 randomly generated problems and compare the  multi-penalty framework to commonly used compressive sensing methods. In particular, we compare to orthogonal matching pursuit (OMP, \cite{tropp2007signal}), $\ell_1$-regularization realized by the Lasso-path algorithm (LASSO, \cite{tibshirani2013lasso}), the basic iterative hard thresholding method \cite{blumensath2009iterative} with a warm start (L1IHT, \cite{arfopeXX})  and the preconditioned Lasso-path algorithm (pLASSO, \cite{jiaro15}). 

To compare the performance and not worry about model selection for other decoders, we assume that a support-size oracle of the solution is given. Thus, for OMP and L1IHT, we have a single support candidate than can be assessed against the true support $I^\dagger = \supp{u^\dagger}$. For the LASSO, and pLASSO, we may in some cases obtain multiple supports, and we choose the closest fit   to $I^\dagger$ according to the symmetric difference to assess their performance. For the multi-penalty framework, we obtain in general several supports for the prescribed support size. To get detailed insights into the performance, we thus record two results. First, we check the theoretical upper performance limit by choosing the support that is closest to $I^\dagger = \supp{u^\dagger}$. These results are labeled as MPLASSO (All), and they confirm and extend experiments that have been conducted in \cite{naumpeter}. Secondly, we record a more realistic performance limit by choosing a support according to the rule
\begin{equation}
\label{eq:selection_criterion}
I^* = \argmax\left\{\frac{\min\limits_{i \in \tilde{I}}|u_{\tilde{I}}|}{||v_{\tilde{I}}||_{\infty}} : \# \tilde{I} = s \right\},
\end{equation}
where the maximum is taken over all supports with matching support size $s$ in the respective solution path, and $u_{\tilde{I}}$, $v_{\tilde{I}}$ are obtained by the least-squares regressions
\begin{equation*}
u_{\tilde{I}} = \argmin\limits_{u \in \mathbb{R}^{\# \tilde{I}}} \left\Vert A_{\tilde{I}} u - y\right\Vert_2^2 \text{ and } 
v_{\tilde{I}} = A^\dagger(y - A_{\tilde{I}} u_{\tilde{I}}).
\end{equation*}
Here, $A^\dagger$ is the pseudo-inverse of $A$. 

Using this criterion, the regularization parameters are chosen adaptively to the given data. These results are labeled as MPLASSO (Rank).  Let us stress here that a specific problem may allow to define better suited selection criteria than \eqref{eq:selection_criterion}. The advantage of our method is that any such criterion can be evaluated on the entire support tiling, making data-driven parameter choices possible.

We test the accuracy and efficiency of the considered methods with respect to the sparsity level, dimension of the problem, and different levels of the two types of signal noise.  In addition to these comparisons, we provide a clear indication on the importance of the proper choice of $\beta$ for multi-penalty regularization by comparing the success rate achieved with adaptively chosen parameters versus success rate achieved when $\beta$ is fixed and only $\alpha$ is chosen adaptively.

\paragraph{Data} For our experiments we will consider the following settings:
\begin{itemize}
\item The $\beta$-range in which the tiling is calculated is fixed to $(10^{-6}, 100)$ since we did not observe any changes in the support tiling for larger $\beta$-ranges.
\item The signal is a vector
$u^\dagger$ containing entries whose absolutes are uniformly sampled from a range $(c_{\textrm{min}}, c_{\textrm{max}}), c_{\textrm{min}} = 1.5$ and $c_{\textrm{max}} = 5$. A distinct entry is chosen at random and set to $c_{\textrm{min}}$ to ensure that the minimum is being taken. The sign pattern for the vector is created at random afterwards.
\item For the signal noise $v$, we sample entries from a uniform distribution on $[-0.2,0.2]$. 
\item For each configuration, we perform $100$ experiments and denote the averaged results.
\item The noise vector $\delta$ is a Gaussian random vector such that $\Vert \delta\Vert_2/\Vert y \Vert_2 = \sigma$. The default level is $\sigma = 0.02$, except when we explicitly change it in the experiments.
\end{itemize}

We present experiments with three types of measurement operators. The first type are Gaussian matrices,  created by drawing entries from a standard normal distribution and subsequently rescaling each column by $1/\sqrt{m}$.  The second type are partial random circulant matrices \cite{gray2006toeplitz},  which essentially reflect sampling processes in many practical applications where sampling processes is modeled by convolution with a random pulse \cite{FoRa13}. Such matrices are created by first taking a Rademacher sequence $b=(b_1,\ldots,b_n) \in \{\pm 1\}^n$ and afterwards creating the related circulant matrix $\tilde{A}_{ij} = b_{j-i \mod N} \in \mathbb{R}^{n \times n}$. 
The sensing matrix is then obtained by choosing $m$ columns of $\tilde{A}$ at random and rescaling the columns by $1/\sqrt{m}$.
The third type are Gamma/Gaussian matrices \cite{jiaro15}, which were used to present the problems of the pLasso \cite{jiaro15} for sampling operators whose singular values are spread over a far greater range, as compared to the other two measurement operators. The matrices are simulated as $A_{ij} = (G_i \backslash\alpha)Z_{ij}$ , where $Z_{ij}$ are normally distributed random variables and the $G_i$ are independent Gamma random variables with shape one and rate one. 

In order to assess the obtained results, we measure the  success rate (whether the correct support $I^\dagger$ is exactly attained by a specific method), as well as the number of elements in the symmetric difference (SD) by $\# (\supp(u) \Delta \supp(u^\dag)).$

\paragraph{Recovery rates for the varying support size} Figure \ref{fig:gaussian_exp1} shows the accuracy of the support recovery by different methods in case of Gaussian matrices. The experiments show that the multi-penalty framework (both with selection criterion and with perfect selection) always performs better than the single-penalty counterparts and even slightly better than pLasso. 
This supports the claim that the Lasso and pLasso support paths are incorporated in the multi-penalty solution space, i.e., that there exists a $\beta$ for which the multi-penalty approach resembles Lasso or the preconditioned version as described in Remark \ref{re:plasso}.
At the same time, the OMP achieves the best performance among all methods. If $A$ is random circulant matrix or Gamma/Gaussian matrix as illustrated in Figures \ref{fig:random_exp1} and \ref{fig:gamma_exp1}, we can observe a worse performance of OMP and other methods, whereas MPLASSO (All) and MPLASSO (Rank) demonstrate a superior performance.    
These results indicate that the multi-penalty framework is not influenced by different sampling operators as it is observed for other methods.\footnote{We ran equal experiments also for random Toeplitz matrices that are related to partial random circulant matrices and obtained similar results with respect to the performance of all sparse encoders. See  \url{https://github.com/soply/mp_paper_experiments} for the results.}

The presented results demonstrate that the performance of pLASSO deteriorates significantly for $A$ being Gamma/Gaussian matrix. This is because the non-zero part of the spectrum of Gamma/Gaussian matrices is spread over a far greater range than for the other two matrix types. The authors \cite{jiaro15} mention that pLasso also performs poorly if a Gaussian matrix is of dimension $m \approx n$ since the distribution of the spectrum follows the Marchenko-Pastur law with mass around zero. Thus, the results indicate a certain stability of the multi-penalty framework against various types of spectra of the sampling operator.

\begin{figure}[t]
\centering
    \begin{subfigure}[b]{0.495\textwidth}
        \includegraphics[width=\textwidth]{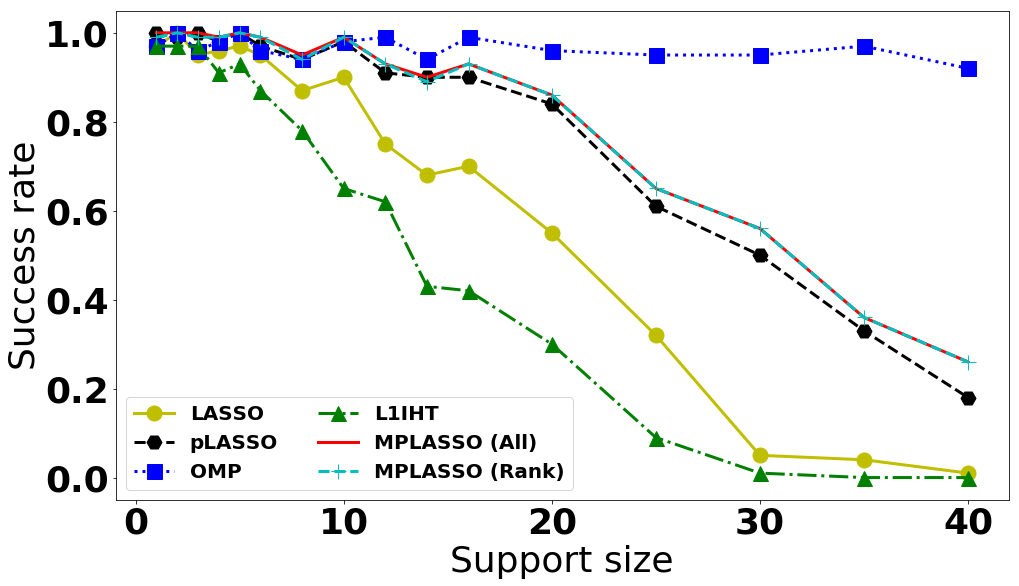}
        \caption{}
        \label{fig:gaussian_success_vs_snr_1}
    \end{subfigure}
    \begin{subfigure}[b]{0.495\textwidth}
        \includegraphics[width=\textwidth]{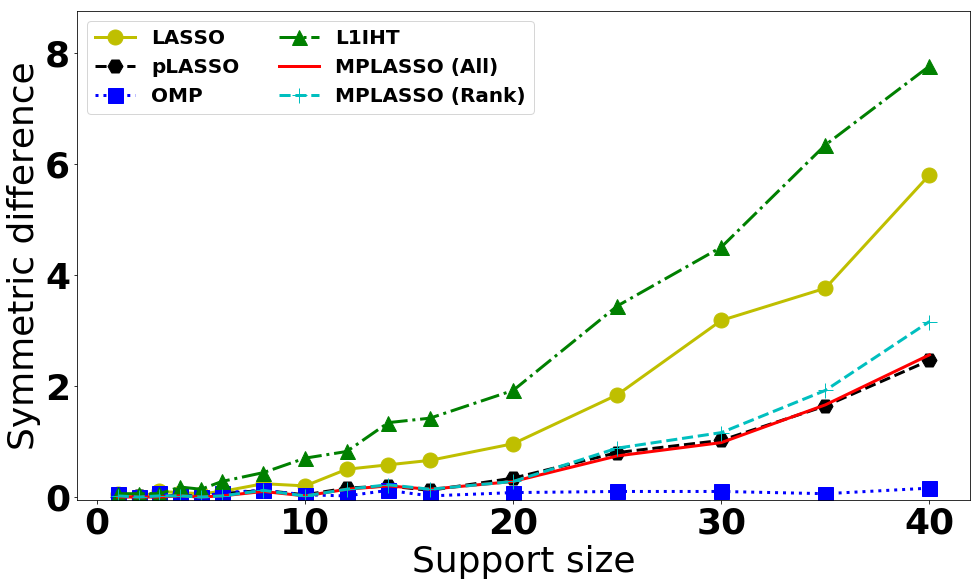}
        \caption{}
        \label{fig:gaussian_success_vs_signal_gap_1}
    \end{subfigure}
\caption{Accuracy of the support recovery for Gaussian random matrices $A \in \R^{600\times 2500}$ and varying support sizes $s$: (a) success rate (b) symmetric difference.}
\label{fig:gaussian_exp1}
\end{figure}

\begin{figure}[t]
\centering
    \begin{subfigure}[b]{0.495\textwidth}
        \includegraphics[width=\textwidth]{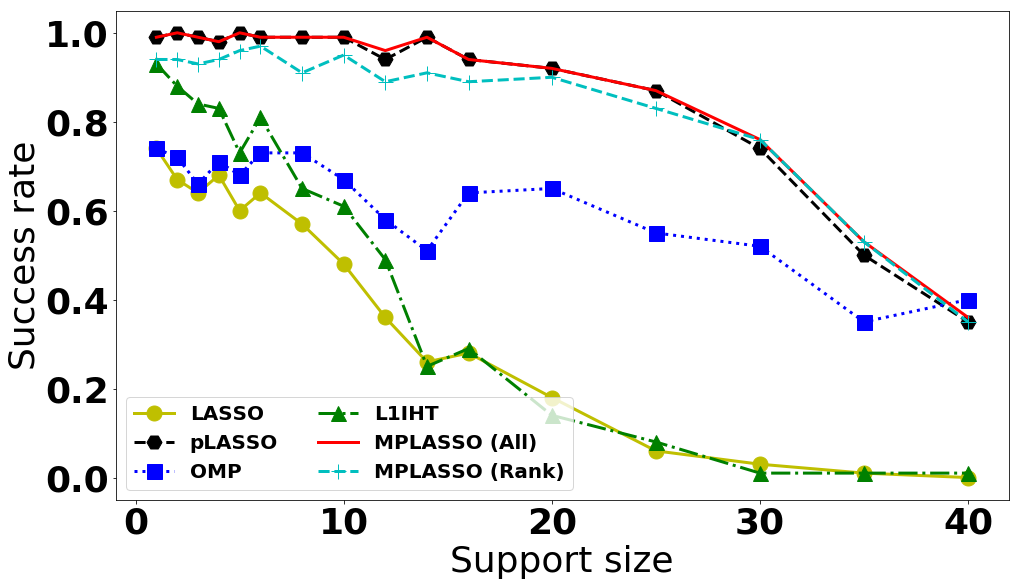}
        \caption{}
        \label{fig:random_success_vs_snr_1}
    \end{subfigure}
    \begin{subfigure}[b]{0.495\textwidth}
        \includegraphics[width=\textwidth]{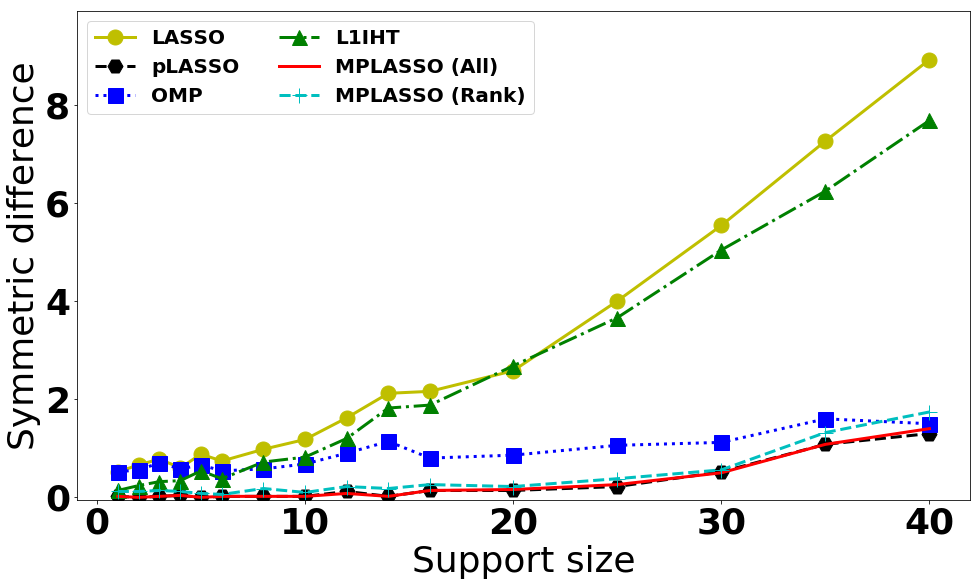}
        \caption{}
        \label{fig:random_success_vs_signal_gap_1}
    \end{subfigure}
\caption{Accuracy of the support recovery for random circulant matrices $A \in \R^{900\times 2500}$ and varying support sizes $s$: (a) success rate (b) symmetric difference.}
\label{fig:random_exp1}
\end{figure}

\begin{figure}[t]
\centering
    \begin{subfigure}[b]{0.495\textwidth}
        \includegraphics[width=\textwidth]{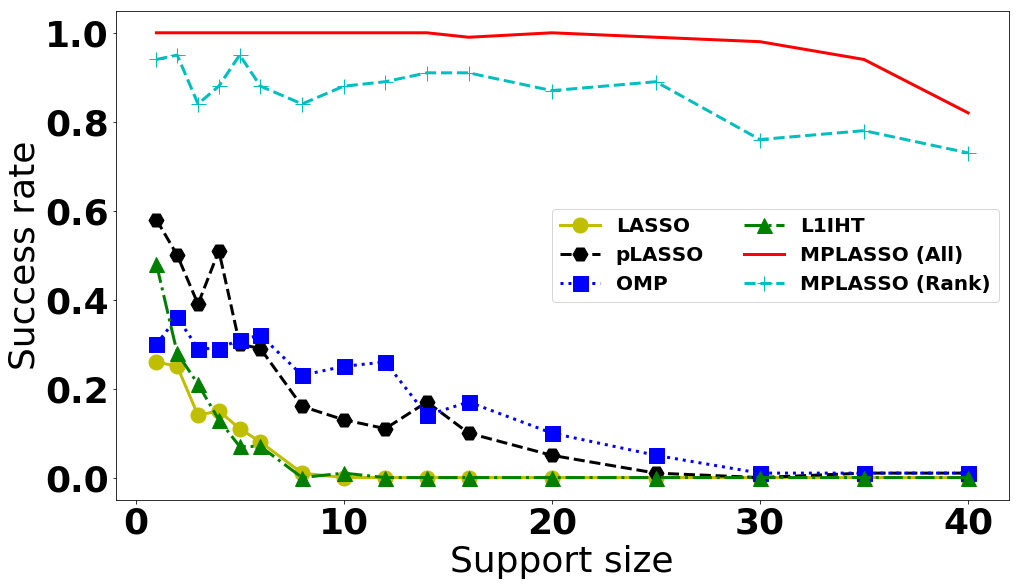}
        \caption{}
        \label{fig:gamma_success_vs_snr_1}
    \end{subfigure}
    \begin{subfigure}[b]{0.495\textwidth}
        \includegraphics[width=\textwidth]{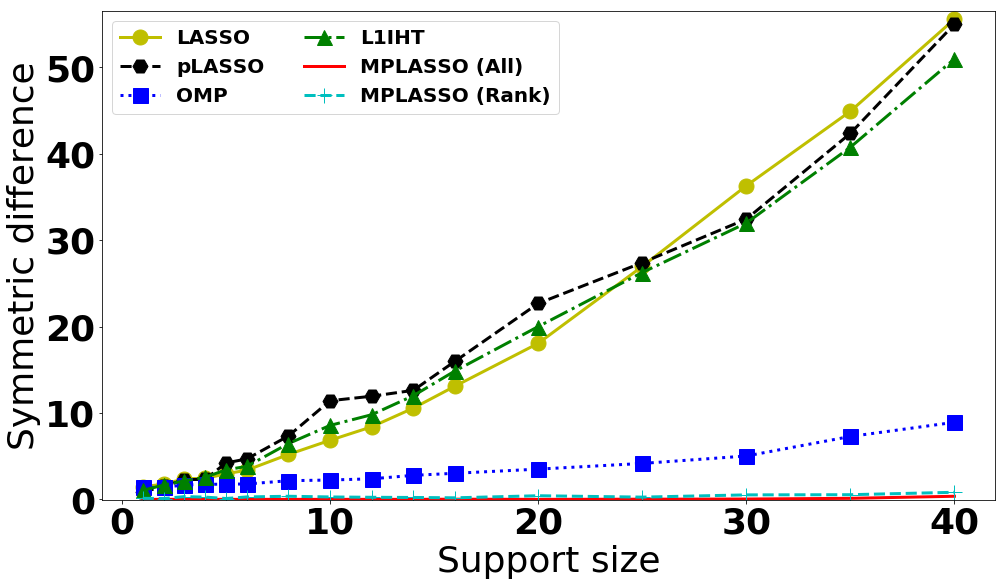}
        \caption{}
        \label{fig:gamma_success_vs_signal_gap_1}
    \end{subfigure}
\caption{Accuracy of the support recovery for Gamma/Gaussian matrices $A \in \R^{900\times 2500}$ and varying support sizes $s$: (a) success rate (b) symmetric difference.}
\label{fig:gamma_exp1}
\end{figure}

\paragraph{Recovery rates for the varying dimension} In these experiments, we study the influence of dimensionality on the reconstruction accuracy. Specifically, for three types of matrices we fix the number of measurements to $m =250$ and the sparsity size $s =20$, while varying $n$ as $n = 2^k$ for $k = 5, \ldots, 15$.  As indicated in Figures \ref{fig:gaussian_exp2}--\ref{fig:gamma_exp2}, we essentially observe a similar performance as for the varying support sizes. In the case of Gaussian random matrices, OMP is the best performing method. If we switch to other sampling operators, OMP is either strongly dependent on the matrix dimension (partial random circulant matrices) or its performance drops quickly (Gamma/Gaussian). Across all methods, we see that MPLASSO is most consistent with respect to different types of sampling operators, and dimensions of the matrices. Note that we do not observe the performance drop of pLasso for almost square Gaussian matrices here. This might be because the default noise level $\delta = 0.02$ is too small to severely deteriorate the performance. We refer to Section 4.1 of \cite{jiaro15} to see this phenomenon.

\begin{figure}[t]
\centering
    \begin{subfigure}[b]{0.495\textwidth}
        \includegraphics[width=\textwidth]{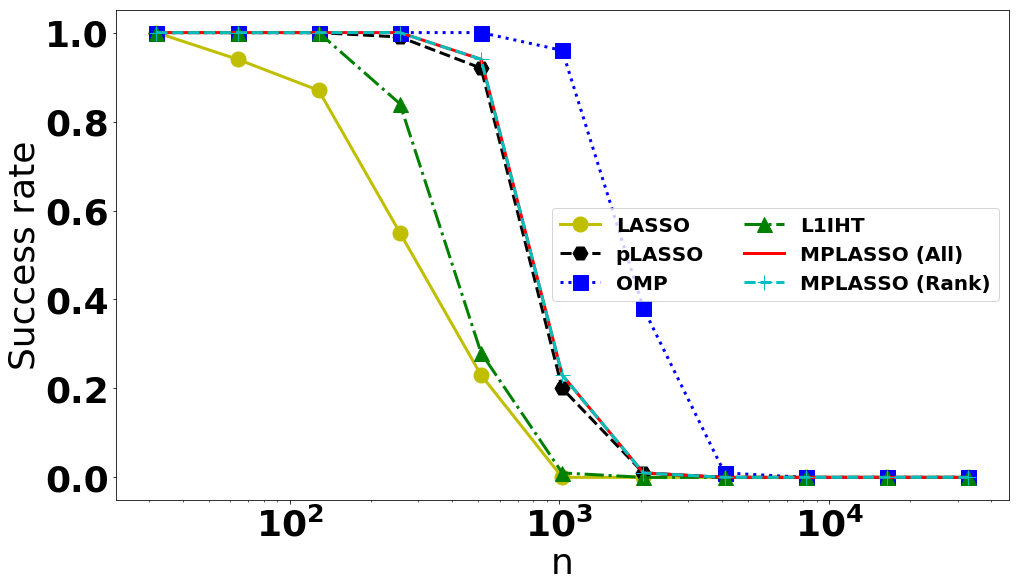}
        \caption{}
        \label{fig:gaussian_success_vs_dim_1}
    \end{subfigure}
    \begin{subfigure}[b]{0.495\textwidth}
        \includegraphics[width=\textwidth]{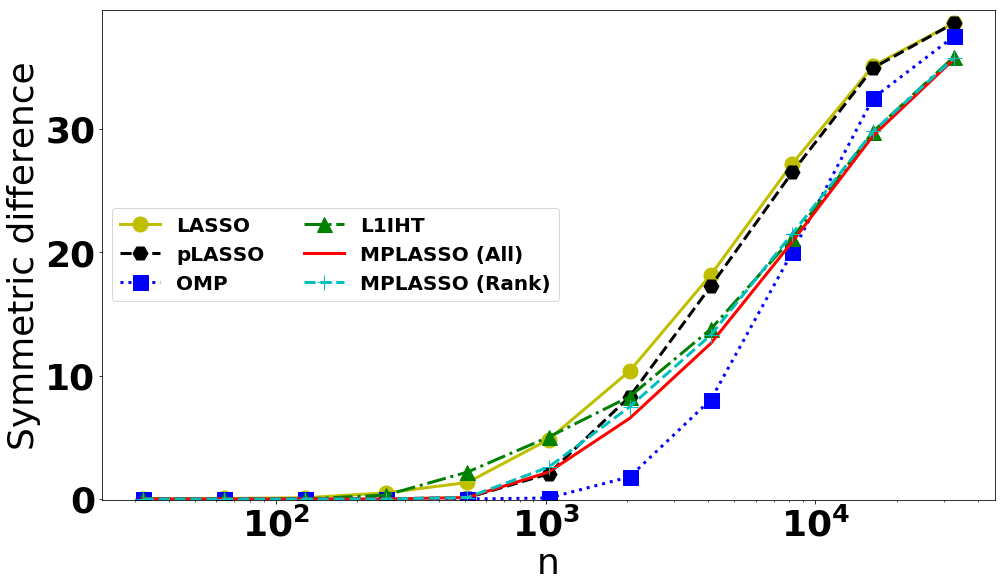}
        \caption{}
        \label{fig:gaussian_success_vs_dim_1}
    \end{subfigure}
\caption{Accuracy of the support recovery for Gaussian random matrices $A \in \R^{250 \times n}$ for $n = 32,\ldots, 32768$: (a) success rate (b) symmetric difference.}
\label{fig:gaussian_exp2}
\end{figure}

\begin{figure}[t]
\centering
    \begin{subfigure}[b]{0.495\textwidth}
        \includegraphics[width=\textwidth]{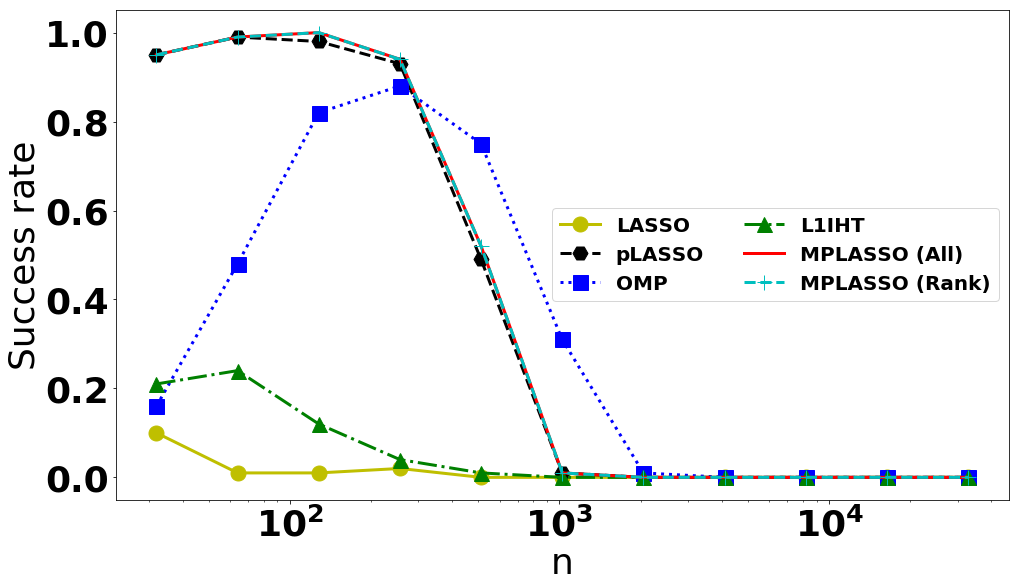}
        \caption{}
     \end{subfigure}
    \begin{subfigure}[b]{0.495\textwidth}
        \includegraphics[width=\textwidth]{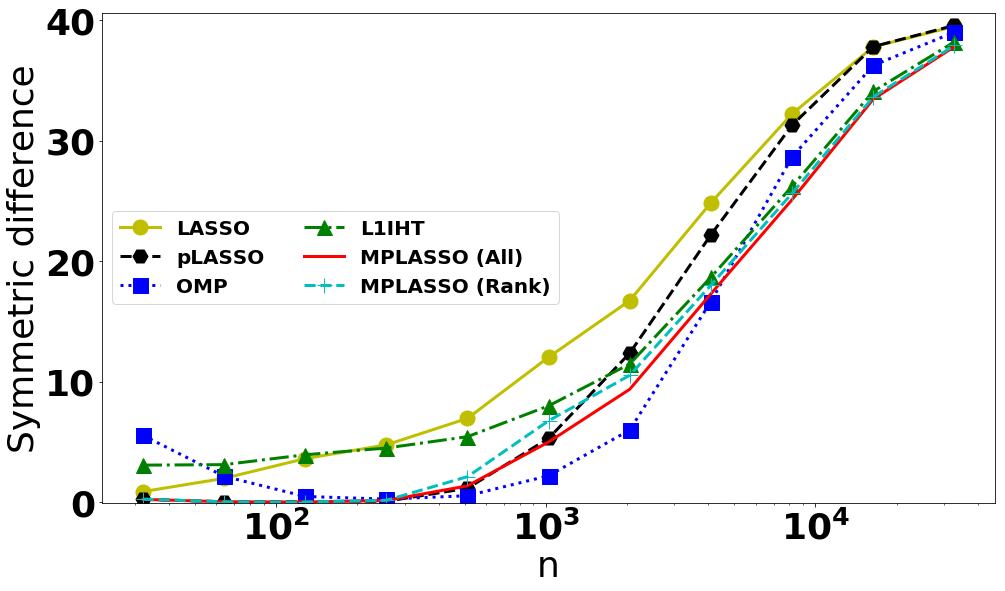}
        \caption{}
    \end{subfigure}
\caption{Accuracy of the support recovery for random circulant matrices $A \in \R^{250 \times n}$ for $n = 32,\ldots, 32768$: (a) success rate (b) symmetric difference.}
\label{fig:random_exp2}
\end{figure}

\begin{figure}[t]
\centering
    \begin{subfigure}[b]{0.495\textwidth}
        \includegraphics[width=\textwidth]{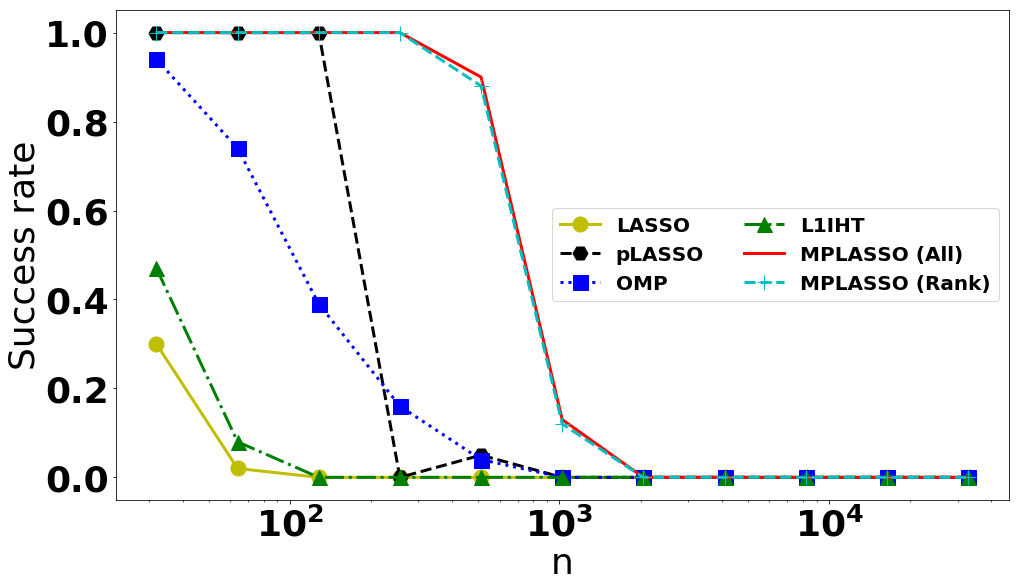}
        \caption{}
    \end{subfigure}
    \begin{subfigure}[b]{0.495\textwidth}
        \includegraphics[width=\textwidth]{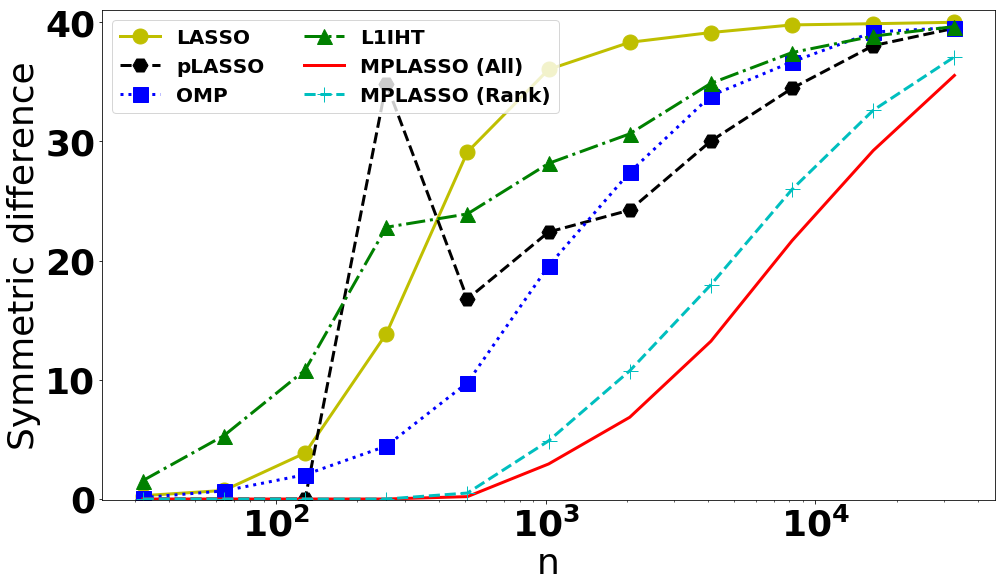}
        \caption{}
    \end{subfigure}
\caption{Accuracy of the support recovery for Gamma/Gaussian matrices $A \in \R^{250 \times n}$ for $n = 32,\ldots, 32768$: (a) success rate (b) symmetric difference.}
\label{fig:gamma_exp2}
\end{figure}

\paragraph{Recovery rates for the varying noise} Finally, we investigate the robustness of all methods with respect to measurement noise. In these experiments, we fix the support size $s =15$ and the matrix sizes as in the first set of experiments, while varying $\sigma \in [0,0.1]$. We again observe a similar performance as in the previous experiments though the multi-penalty framework performs similarly to OMP also for Gaussian matrices, see Figures \ref{fig:gaussian_exp3}--\ref{fig:gamma_exp3}. The observed pattern is retained by increasing the signal noise, i.e., the multi-penalty regularization as well as pLASSO have a superior performance compared to all other methods.

It is worth mentioning here that the results show that the noise level
$\sigma$ has almost no influence on the performance of the different methods
in case of Gaussian of random circulant matrices. This might be due to the fact
that these matrices are essentially well-conditioned in the sense that its
singular values are bounded away from zero and well concentrated.
As a consequence, we can expect that there is almost no directional dependance of the distribution
of the noise vector $Av$. Adding the relatively small noise vector $\delta$ will thus only 
have a rather small influence on the total noise level.

In contrast, the Gamma/Gaussian matrices have very small singular values,
which implies that there will be directions in which the multiplication
with $A$ concentrates the uniformly distributed noise vector $v$ very
tightly around zero. In these directions, the addition of $\delta$ is
noticeable already for relatively small variances $\sigma$, and thus
the influence of $\sigma$ on the results is stronger.

\begin{figure}[t]
\centering
    \begin{subfigure}[b]{0.495\textwidth}
        \includegraphics[width=\textwidth]{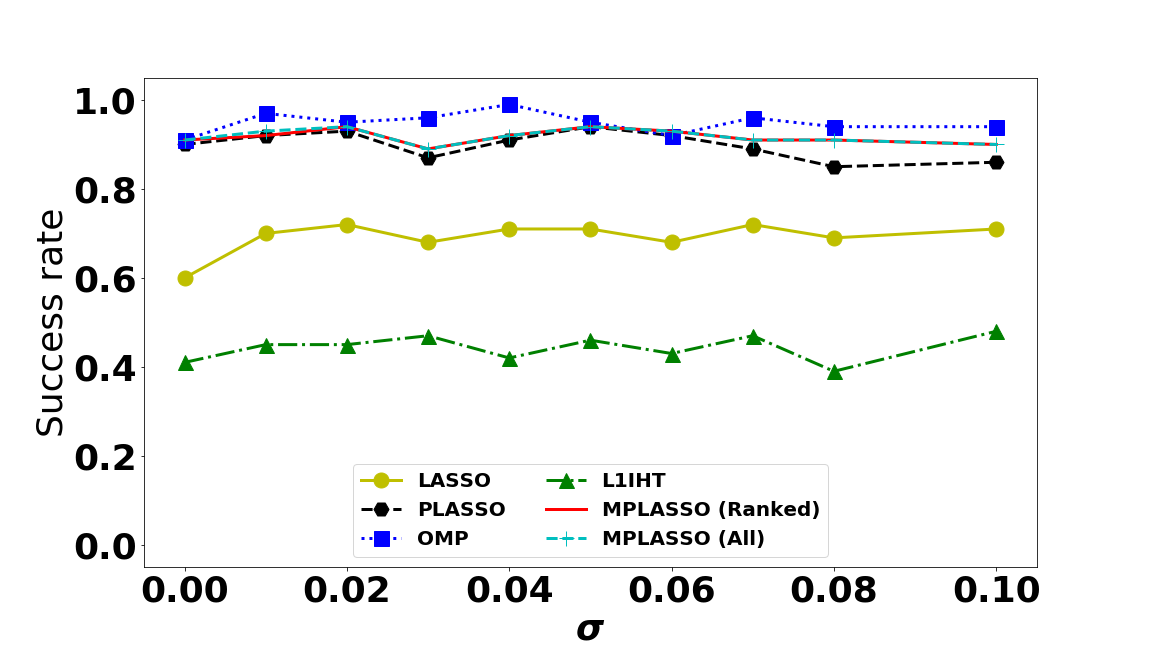}
        \caption{}
    \end{subfigure}
    \begin{subfigure}[b]{0.495\textwidth}
        \includegraphics[width=\textwidth]{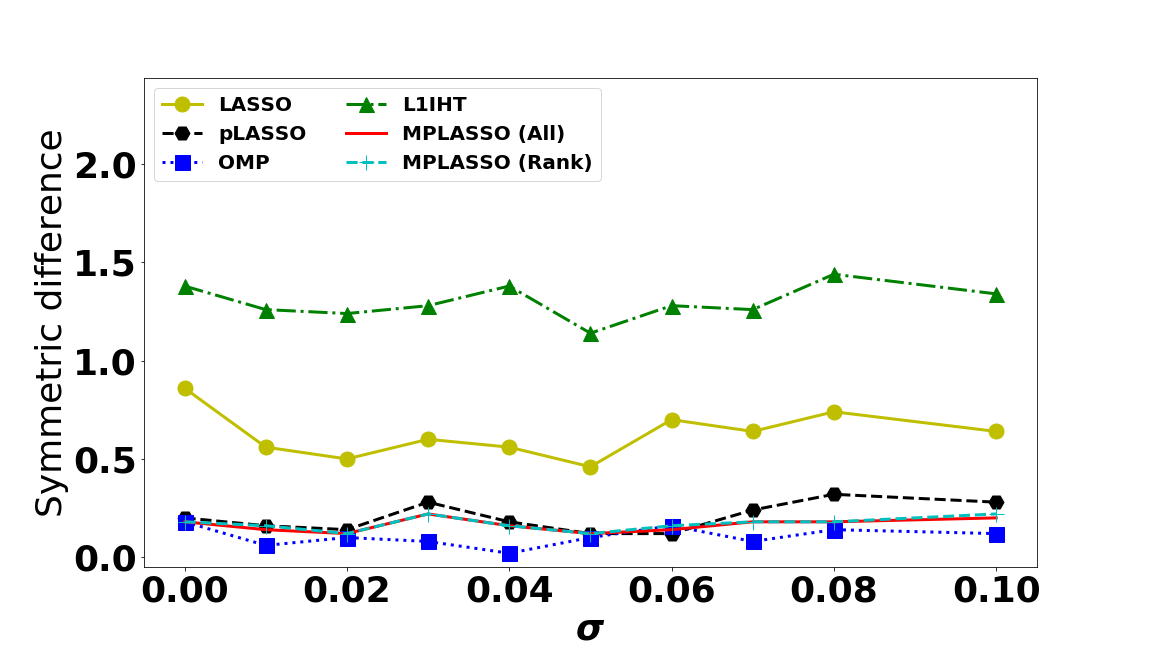}
        \caption{}
    \end{subfigure}
\caption{Accuracy of the support recovery for Gaussian random matrices $A \in \R^{600\times 2500}$ and varying measurement noise $\sigma$: (a) success rate (b) symmetric difference.}
\label{fig:gaussian_exp3}
\end{figure}

\begin{figure}[t]
\centering
    \begin{subfigure}[b]{0.495\textwidth}
        \includegraphics[width=\textwidth]{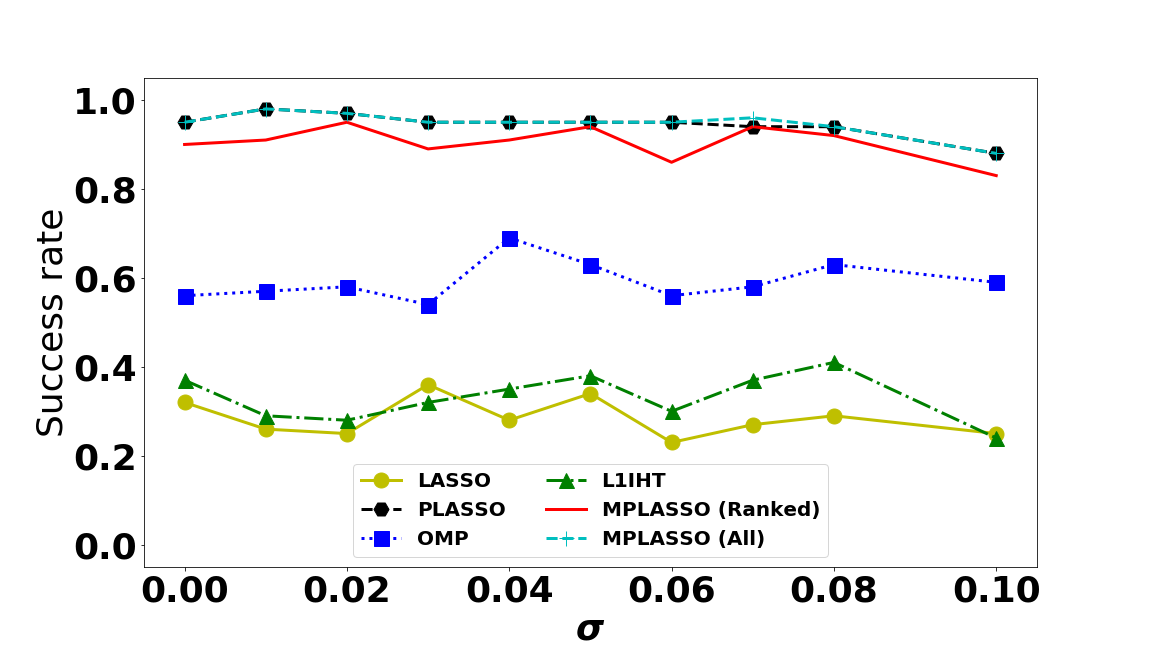}
        \caption{}
     \end{subfigure}
    \begin{subfigure}[b]{0.495\textwidth}
        \includegraphics[width=\textwidth]{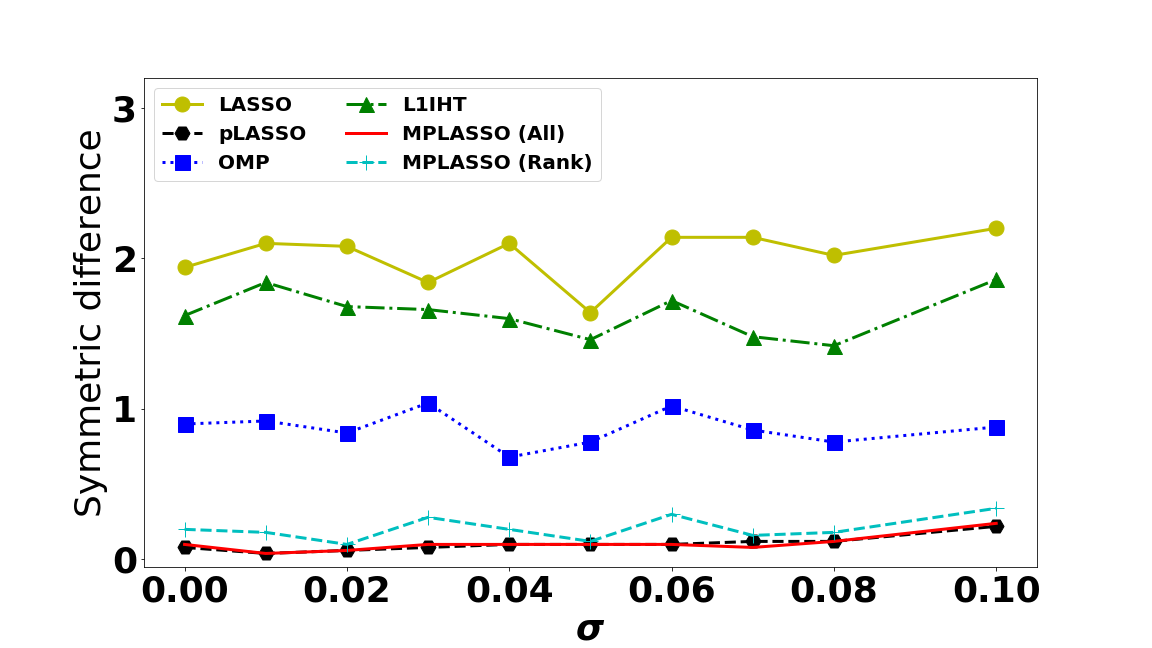}
        \caption{}
     \end{subfigure}
\caption{Accuracy of the support recovery for random circulant matrices $A \in \R^{900\times 2500}$ and varying measurement noise $\sigma$: (a) success rate (b) symmetric difference.}
\label{fig:random_exp3}
\end{figure}

\begin{figure}[t]
\centering
    \begin{subfigure}[b]{0.495\textwidth}
        \includegraphics[width=\textwidth]{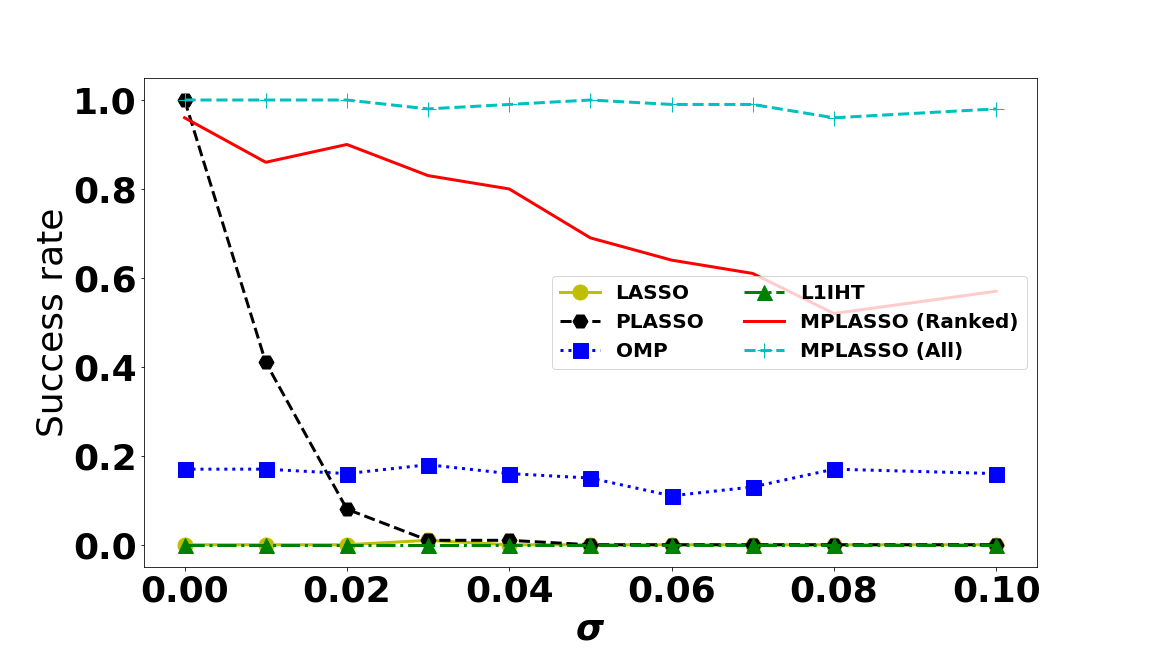}
        \caption{}
    \end{subfigure}
    \begin{subfigure}[b]{0.495\textwidth}
        \includegraphics[width=\textwidth]{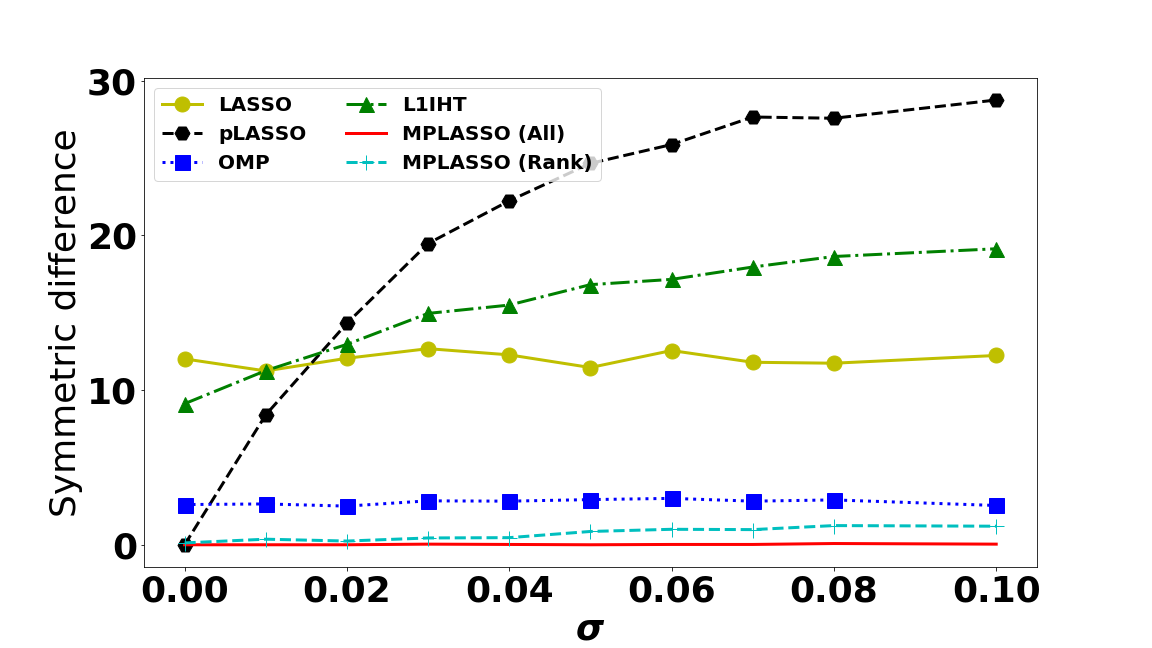}
        \caption{}
    \end{subfigure}
\caption{Accuracy of the support recovery for Gamma/Gaussian matrices $A \in \R^{900\times 2500}$ and and varying measurement noise $\sigma$: (a) success rate (b) symmetric difference.}
\label{fig:gamma_exp3}
\end{figure}

\paragraph{Recovery rates for fixed/adaptively chosen $\beta$} Apart from comparing the introduced multi-penalty framework to state-of-the-art regularization methods, we also compare it to the multi-penalty regularization with a fixed $\beta$, or, in other words, we study the necessity of an adaptive $\beta$ choice for the optimal performance of the method. In Figure \ref{fig:optimal_beta} (a) and (b), we show in horizontal lines the upper performance limit given by the result of MPLASSO (All) for a specific experiment. The curves indicate the performance of the multi-penalty method using a fixed $\beta$-choice, that is ranging on the $x$-axis.

In Figure \ref{fig:optimal_beta} (a), we see that an adaptive choice is essentially not necessary because the multi-penalty method with very small $\beta \approx 0$ works quite well across all noise levels. This matches our previous results since we saw that pLASSO and MPLASSO perform similar on partial random circulant matrices with dimensions $m \ll n$. For Gamma/Gaussian matrices in \ref{fig:optimal_beta} (b) however, we observe that an adaptive choice is necessary. There exists no single $\beta$ that yields the upper performance limit across all noise levels. Even if the noise level is fixed and sufficiently large, there is no $\beta$ that reaches the upper limit for all realizations of the experiment.

The observed results are similar for different support size and matrix dimensions, and thus
we only presented one specific matrix configuration. The results of further numerical
experiments can be found in the Jupyter notebooks.

\begin{figure}[t]
\centering
    \begin{subfigure}[b]{0.49\textwidth}
        \includegraphics[width=\textwidth]{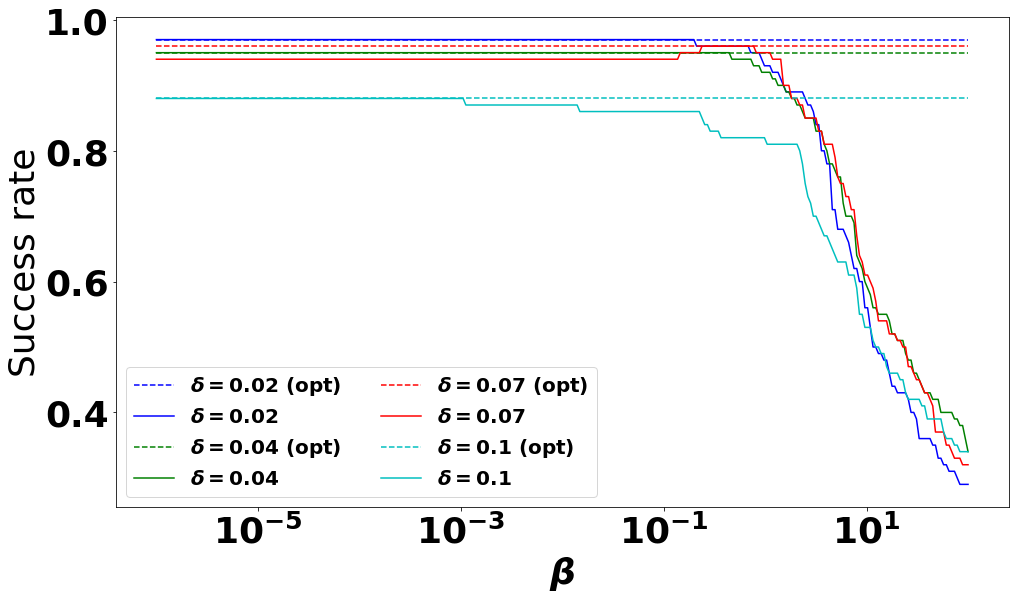}
        \caption{}
    \end{subfigure}
    \begin{subfigure}[b]{0.49\textwidth}
        \includegraphics[width=\textwidth]{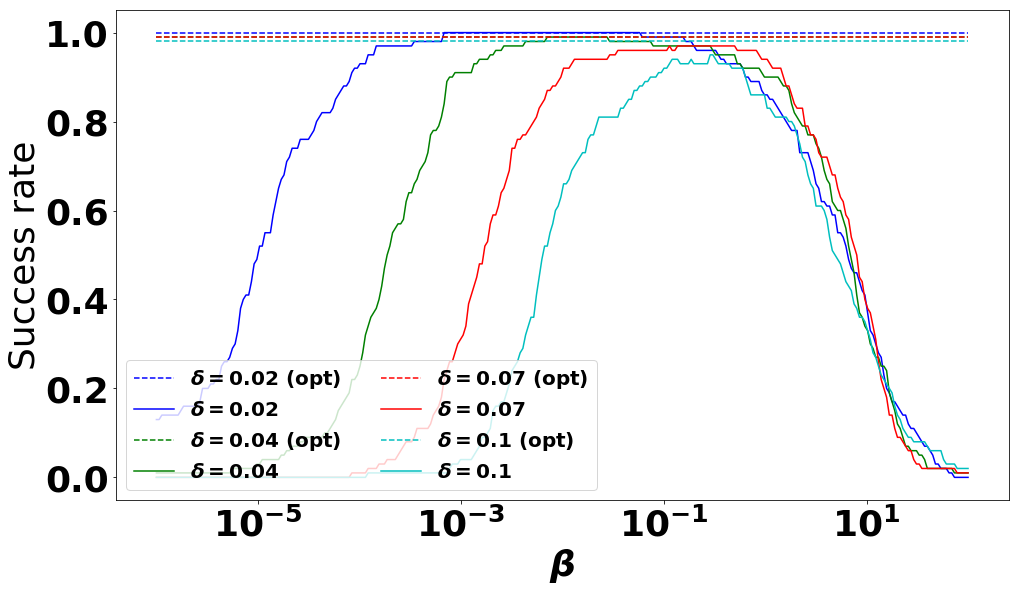}
        \caption{}
    \end{subfigure}
\caption{Success rate of the support recovery for varying noise levels and fixed versus adaptively chosen $\beta$. The horizontal line shows the upper performance limit, extracted from the results of MPLASSO (all), while the curve shows the performance if we choose a fixed $\beta$ across all experiments. $\beta$ is ranging on the $x$-axis. (a) random circulant matrices $A \in \R^{900\times 2500}$; (b) Gamma/Gaussian matrices  $A \in \R^{900\times 2500}$.}
\label{fig:optimal_beta}
\end{figure}

\medskip

Overall, our experiments show that the multi-penalty framework yields a robust method to solve general compressed sensing problems. In cases where $\ell_1$-based methods like Lasso are preferred over greedy methods such as OMP, the multi-penalty framework consistently yields a better performance than conventional alternatives. Admittedly, this is expected because the solution spaces of Lasso are incorporated into multi-penalty functional for $\beta \rightarrow \infty$. The experiments confirm that the introduced criterion for the correct support selection works well in many cases, such that our algorithm allows to consistently achieve equal or better performances than other sparse decoders. Only in case of a Gaussian sampling matrix, the greedy OMP is superior to the multi-penalty framework.

Consequently, the multi-penalty framework in combination with the support selection \eqref{eq:selection_criterion} is a reasonable approach to solving compressed sensing problems where additional signal noise $v$ affects the signal $u$ before the sampling procedure.

\section{Conclusion and future directions}
\label{sec:discussion}

Inspired by a challenging problem of support recovery and building upon the recent advances in regularization theory, signal processing, and statistics, we have presented a novel algorithmic framework for finding a solution of the unmixing problem by means of multi-penalty regularization. Unlike classical approaches for parameter learning, where the choice is made by discretizing the parameter space, we first compute all possible sufficiently sparse solutions, attainable from the given datum $y$, and then apply standard regression for the accurate reconstruction of the non-zero components. The advantage of this framework is that we obtain an overview of the solution stability over the entire range of parameters, which can be used for obtaining insights into the problem or to investigate other parameter learning rules. 

We show and exemplify by experiments that our method can be interpreted as interpolation between the standard and pre-conditioned Lasso. Therefore, the multi-parameter approach combines the advantages of both single-penalty regularization methods and mitigates their drawbacks. In particular, as demonstrated in the extensive experiments with different measurement operators, our algorithm outperforms its single-penalty counterparts for sufficiently sparse solutions and is cost-efficient for supports of medium size. Moreover, our method outperforms the greedy counterpart OMP for the random circulant and Gamma/Gaussian matrices. The proposed method shows robustness and stability against measurement and signal noise, whereas all other methods are not consistent across all settings.

We plan to investigate strategies for speeding up our method and, thus, making it computationally more efficient. In particular, we plan to introduce an adaptive discretization to allow for cost reduction in computing $B_\beta$ for various $\beta$ and support sizes. 
Moreover,  we plan to extend the approach to other types of signals and noise such as bounded variation signals, also considering higher-dimensional signals such as 2D and 3D images. 

All experiments can be reproduced using the freely available source code available in the Github repositories\footnote{See \url{https://github.com/soply/sparse_encoder_testsuite} and \url{https://github.com/soply/mpgraph}}. The code is well-documented, therefore we refer for further information to these repositories to run the presented experiments.

\end{document}